\documentclass{amsart}

\usepackage{amsfonts}
\usepackage{amsmath}
\usepackage{enumerate}
\usepackage{tcolorbox}
\usepackage{amsmath,amssymb}

\usepackage{mathrsfs} 
\usepackage{subfigure}

\usepackage{algorithm,algpseudocode}
\usepackage{setspace} 

\usepackage{mathtools} 

\usepackage{tikz}
\usetikzlibrary{trees}

\usepackage{booktabs} 
\usepackage{multirow}
\usepackage{array} 
\usepackage{hyperref}
\hypersetup{
    colorlinks,
    citecolor=black,
    filecolor=black,
    linkcolor=black,
    urlcolor=black
}

\NewEnviron{mblock}[1][colback=yellow]{%
\begin{tcolorbox}
\BODY
\end{tcolorbox}
}

\tcbuselibrary{most}
\newtcolorbox{mybox}[2][]
{colback = red!5!white, colframe = red!75!black, fonttitle = \bfseries,
colbacktitle = red!85!black, enhanced,
attach boxed title to top center={yshift=-2mm},
title=#2,#1}

\usepackage{multirow}

\newtheorem{theorem}{Theorem}[section]
\newtheorem{lemma}[theorem]{Lemma}
\newtheorem{proposition}[theorem]{Proposition}
\newtheorem{definition}[theorem]{Definition}
\newtheorem{corollary}[theorem]{Corollary}
\newtheorem{example}[theorem]{Example}
\newtheorem{assumption}[theorem]{Assumption}
\newtheorem{remark}[theorem]{Remark}

\numberwithin{equation}{section}

 \def\M{\mathcal{M}} 
\def\Z{\mathcal{Z}} 
\def\E{\mathcal{E}} 
 
\def\innormd{\mathcal{Z}} 
\def\exnormd{\mathcal{X}}

\newcommand{\parallelslant}{\mathbin{\!/\mkern-5mu/}}

\title{Accelerate Vector Diffusion Maps by Landmarks}

\author[M.-P. Tsui]{Mao-Pei Tsui}
\address{Department of Mathematics, National Taiwan University, Taipei, 106, Taiwan;
National Center for Theoretical Sciences, Mathematics Division, Taipei, 106, Taiwan}
\email{maopei@math.ntu.edu.tw}

\author[H.-T. Wu]{Hau-Tieng Wu}
\address{Courant Institute of Mathematical Sciences, New York University, New York, NY, 10012 USA}
\email{hauwu@cims.nyu.edu}

\author[Y.-A. Wu]{Yi-An Wu}
\address{Data Science Degree Program, National Taiwan University and Academia Sinica, Taipei, 106, Taiwan}
\email{d07946009@ntu.edu.tw}

\author[S.-Y. Yeh]{Sing-Yuan Yeh}
\address{National Central University, Taoyuan, 320, Taiwan}
\email{syyeh@math.ncu.edu.tw}

\begin{document}

\begin{abstract}
We propose a landmark-constrained algorithm, LA-VDM (Landmark Accelerated Vector Diffusion Maps), to accelerate the Vector Diffusion Maps (VDM) framework built upon the Graph Connection Laplacian (GCL), which captures pairwise connection relationships within complex datasets. LA-VDM introduces a novel two-stage normalization that effectively address nonuniform sampling densities in both the data and the landmark sets. Under a manifold model with the frame bundle structure, we show that we can accurately recover the parallel transport with landmark-constrained diffusion from a point cloud, and hence asymptotically LA-VDM converges to the connection Laplacian. The performance and accuracy of LA-VDM are demonstrated through experiments on simulated datasets and an application to nonlocal image denoising.
\end{abstract}

\maketitle


\section{Introduction}

Modern datasets typically consist of multiple attributes, and the relationships among data points are often highly complex. Such complexity presents significant challenges in defining meaningful notions of distance or metric, and even in quantifying the underlying nonlinear relationships. The ultimate goal of data analysis is to exploit not only the metric structure but also the underlying nonlinear relationships, with the aim of designing algorithms that are both theoretically sound and interpretable for practical applications.

As an illustrative example, consider a dataset of clean images, $I_i\in L^2(\mathbb{R}^2)$, $i=1,\ldots,n$, where one image is a rotated version of another; that is, $I_i=R_{ij}\circ I_j$, where $R_{ij}\in O(2)$ and  $R_{ij}\circ I_j(x)=I_j(R_{ij}(x))$ for all $x\in \mathbb{R}^2$. The standard Euclidean, or $L^2$, distance between $I_i$ and $I_j$ may be large, despite their intrinsic similarity. To address this discrepancy, particularly when intrinsic similarity is the focus, it is natural to account for the nonlinear rotational relationship by introducing the {\em rotationally invariant distance} (RID), defined as $d_{\texttt{RID}}(I_i,I_j):=\min_{O\in O(2)}\|I_i-O\circ I_j\|_{L^2}$. By construction, the RID between any two such images is zero, which more faithfully captures the equivalence of image content regardless of rotation. Similarly, for $f(x)=\sin(x)\in L^2([0,2\pi))$, if we define $T_\theta f(x)=f(x+\theta)$, where $\theta\in [0,2\pi)$, we have $\|T_\theta f-f\|_{L^2}=2\pi$ when $\theta\neq 0$, while their phase-invariant distance (PID), defined as $d_{\texttt{PID}}(f_1,f_2):=\min_{\theta\in [0,2\pi)}\|f_1-T_\theta f_2\|_{L^2}$, is zero.  
By a careful rethink about the RID and PID, we can conclude that the nonlinear relationship could be obtained by the same procedure. For example, the rotational relationship between $I_i$ and $I_j$ can be obtained by $O_{ij}:=\arg\min_{O\in O(2)}\|I_i-O\circ I_j\|_{L^2}$, and similarly for PID. It is therefore natural to ask whether incorporating such latent nonlinear relationships into the analysis can improve data analysis. The short answer is yes, and this perspective has already led to fruitful results \cite{zhao2014rotationally,huroyan2020solving,lin2018,mcerlean2024unsupervised,borik2024graph,ye2017cohomology}.

Mathematically, such latent nonlinear relationships can be modeled through the notion of a connection, which extends the idea of pairwise affinity by incorporating transformation or alignment information between data points. The term ``connection'' is borrowed from differential geometry: when the dataset is sampled from a manifold, the connection coincides with the classical notion used to describe the underlying geometric structure. 
To formalize and leverage such nonlinear structures, the graph Laplacian and diffusion maps \cite{coifman2006} frameworks was extended to the Graph Connection Laplacian (GCL) and Vector Diffusion Maps (VDM) \cite{singer2011} by encoding the connection as a vector-valued relationship among points, which combined with affinity captures the complex structure that we have interest. Under a manifold model, it has been shown that, under mild assumptions, the GCL asymptotically converges to the connection Laplacian \cite{berline2003heat} associated with the manifold's vector bundle structure. Likewise, VDM provides a spectral embedding based on this operator and the associated heat kernel \cite{lin2018embeddings,lin2023manifold}. In the degenerate case where the manifold reduces to a single point, a connected 0-dim manifold, the connection Laplacian encodes solution to the synchronization problem \cite{singer2011angular,bandeira2013cheeger}, among many others, which is an interesting problem attracting lots of attention. The GCL and VDM frameworks, as long as the synchronization, have found widespread applications across diverse domains. Examples include but not exclusively determining the orientability of a manifold using the orientation bundle and recovering the double orientable covering of an unorientable manifold \cite{singer2011orientability}, class averaging in cryo-electron microscopy (cryo-EM) using RID and the frame bundle of a manifold diffeomorphic to $S^2$ \cite{zhao2014rotationally}, phase reconstruction in ptychography (a variant of phase retrieval) using the $U(1)$ bundle \cite{marchesini2016alternating}, nonlocal image denoising using RID and $SO(2)$ bundle \cite{lin2018}, solving jigsaw puzzle problem using RID and $C_4$ bundle \cite{huroyan2020solving}, enhancing the electrocardiogram-derived respiration (EDR) using PID and $SO(2)$ bundle \cite{mcerlean2024unsupervised}, and signal quality enhancement for image-based photoplethysmography (PPG) using PID and $SO(2)$ bundle \cite{borik2024graph}.

Despite their theoretical appeal and empirical success, these methods are computationally demanding. Similar to other spectral algorithms, including the diffusion map (DM), which is a special case of VDM corresponding to the trivial line bundle, the GCL and VDM frameworks rely on eigenvalue decomposition, which typically incurs $O(n^{2.81})$ complexity, where $n$ is the number of sampling points. Such complexity is prohibitive for large-scale datasets and thus restricts their applicability. Several approaches have been proposed to mitigate this issue, including sparsification via $k$-nearest neighbors and the Nystr\"om extension \cite{shen2020}. However, sparsification can be highly sensitive to noise, while the Nystr\"om extension may fail to preserve essential geometric information. See \cite{shen2020} for more discussion.
To address this computational bottleneck, a landmark-based approximation method called  Robust and Scalable Embedding via Landmark Diffusion (ROSELAND) \cite{shen2020,shen2022} was proposed. ROSELAND combines the advantages of sparsification and Nystr\"m extension while significantly reducing computational cost and preseving the robustness of DM. The core idea is simple: diffusion from one point to another is split into two stages, first from the source to all landmarks, and then from the landmarks to the target; that is, constrain the diffusion process through a selected set of landmark points. Both theory and practice demonstrate that this landmark-constrained diffusion closely approximates ordinary diffusion, yet with far greater efficiency. In particular, if the number of landmarks scales as $n^\beta$ with $\beta \in (0,1/2)$, then the computational complexity of  becomes $O(n^{1+2\beta})$, substantially more efficient than traditional approaches. This landmark idea has been applied to speed up kernel based sensor fusion algorithm \cite{yeh2024landmark}.

At first glance, extending ROSELAND to the GCL and VDM frameworks may seem straightforward. However, several challenges arise. First, under the landmark constraint, it is not clear whether one can accurately approximate the connection (or parallel transport), even when assuming a manifold model and access to accurate connection information between data points and landmarks, not to mention when we have only point clouds sampled from the manifold and need to estimate parallel transport of the tangent bundle. 
Second, the performance of ROSELAND depends critically on the landmark sampling density. For instance, if we want to recover the Laplace-Beltrami operator, an appropriate landmark selection or design procedure is needed. Yet, such design is nontrivial, and the only available approach currently relies on accurate estimation of the data density \cite{shen2020}, which is in general a difficult task on the manifold setup. The same density issue persists when extending ROSELAND to GCL and VDM, naturally prompting the question of whether alternative strategies can better address the landmark sampling density problem in recovering the connection Laplacian.

Our contributions in this paper are twofolds. First, we introduce a novel algorithm, LA-VDM (Landmark Accelerated Vector Diffusion Maps), which generalizes ROSELAND to accelerate VDM. The core idea, similar to ROSELAND, is to split diffusion from one point to another into two stages, but with an additional ingredient: each step incorporates an approximation of the connection, thereby preserving connection information under the landmark constraint. 
Another algorithmic novelty is a two-stage normalization scheme designed to address sampling density. The first normalization corrects for the data distribution, while the second handles the landmark distribution. Together, these allow LA-VDM to robustly approximate the connection Laplacian, even under non-uniform sampling of dataset and landmarks. 
Second, we provide an asymptotic analysis of LA-VDM under a general principal bundle framework. At first glance, one might expect that the landmark-constrained paths may introduce nontrivial contributions to the connection estimation due to the well-known path-dependence of parallel transport (see Figure \ref{fig:parallel estimation LAVA}). Surprisingly, we prove that despite this concern, the connection can still be accurately approximated under the landmark constraint and the contribution of landmark-constrained paths is asymptotically of higher order. The key technique is to carefully control the deviation in connection estimation induced by landmark-constrained paths. We also rigorously analyze the role of the two-stage normalizations and validate their impact on the final LA-VDM embedding. Collectively, these results establish that LA-VDM yields an accurate and computationally efficient approximation of the connection Laplacian and thus of VDM. 
Finally, to demonstrate the practical utility of LA-VDM, we validate the algorithm on various simulated datasets. Notably, when the connection is trivial and normalization is omitted, LA-VDM reduces to ROSELAND. Therefore, our proposed two-stage normalization procedure can be applied to improve ROSELAND as well.

The remainder of the paper is organized as follows. In Section~\ref{sec:proposed}, we introduce the LA-VDM algorithm and state its relationship with the vanilla VDM. Section~\ref{sec:thm} provides an asymptotic analysis of LA-VDM under a manifold setting, where we establish two main theorems, Theorems~\ref{thm:bias} and~\ref{thm:var}, which also guide the selection of hyperparameters for LA-VDM. In Section~\ref{sec:simu}, we validate the theoretical results and demonstrate the influence of hyperparameters using various simulation datasets. Finally, Section~\ref{sec:proof} contains the detailed proofs of the main theorems presented in Section~\ref{sec:thm}.


\section{Landmark Accelerated Vector Diffusion Maps (LA-VDM)}\label{sec:proposed}

Suppose we have a point cloud $\tilde{\mathcal{X}} = \{s_i\}_{i=1}^n \subseteq \mathbb{R}^p$. While in some settings background knowledge may be incorporated into the metric, in this work we assume access only to the Euclidean distance. Consider a non-negative kernel function $K: \mathbb{R}_{\geq 0} \to \mathbb{R}_{\geq 0}$ with sufficient regularity and decays sufficiently fast, and Gaussian kernel is a common practical choice.

\subsection{Review VDM}

Before introducing the LA-VDM algorithm, we start with a quick review of VDM \cite{singer2011,singer2015}. Consider an un-directed graph $(\tilde{\mathcal{X}},\mathbb{E})$, where the edge set $\mathbb{E}=\{s_i\}_{i=1}^n$ is constructed from $\tilde{\mathcal{X}}$ following some predetermined rules. 
For example, $\mathbb{E}=\{(s_i,s_j)\}_{i,j=1}^n$ if we want a complete graph or $\mathbb{E}=\{(s_i,s_j)|\|s_i-s_j\|_{\mathbb{R}^p}\leq \epsilon\}$ if we use the $\epsilon$-ball scheme, where $\epsilon>0$. 
To simplify the discussion, no matter how the edge set is constructed, we assume it is symmetric; that is, $(s_i,s_j)\in \mathbb{E}$ if and only if $(s_j,s_i)\in \mathbb{E}$. 
Assign an affinity function $w$ on $\mathbb{E}$ so that $w_{ij}>0$ for each edge $(s_i,s_j)\in\mathbb{E}$. Usually, the affinity $w_{ij}$ is defined by the kernel; that is,
\[
w_{ij}=K\left(\frac{\|s_i-s_j\|_{\mathbb{R}^p}}{\sqrt{\epsilon}}\right)\,,
\]
where $\epsilon>0$ is the chosen bandwidth; that is, two points have small affinity if they are far apart. 
Finally, suppose each vertex is associated with $q\in \mathbb{N}$ attributes of interest. We assign a group-valued function $\Omega:\mathbb{E}\to O(q)$, called a {\em connection function}. While in principle there are many group structures we can consider (see \cite{singer2011}), in this work we focus on the rotation group. We call $\Omega_{ij}$ the {\em connection} between $s_i$ and $s_j$ when $(s_i,s_j)\in \mathbb{E}$. We also assume the symmetric relationship $\Omega_{ji}=\Omega_{ij}^\top$. The connection function depends on the application and typically encodes the nonlinear relationship modeled as a rotation between the attributes at the two vertices. We refer to the quadruple $(\tilde{\mathcal{X}},\mathbb{E}, w,\Omega)$ as a {\em connection graph}.

With the connection graph, construct a $n \times n$ affinity matrix $\mathbf{W}$, so that $\mathbf{W}_{ij}=w_{ij}$ when $(s_i,s_j)\in \mathbb{E}$ and $\mathbf{W}_{ij}=0$ otherwise. 
Construct a diagonal matrix $\mathbf{D}$ as
\begin{equation}
\mathbf{D}_{\mathcal{X}}(i,i) = \sum_{j=1}^n w_{ij}.
\end{equation}
Next, apply the $\alpha$-normalization to normalize to reduce the influence of sampling density by setting
\[
\mathbf{W}_\alpha = \mathbf{D}^{-\alpha} \mathbf{W} \mathbf{D}^{-\alpha},
\]
where $\alpha \in [0,1]$ is a hyperparameter. With $\mathbf{W}_\alpha$, construct an $n\times n$ \textit{affinity-connection} block matrix with $q\times q$ entries by setting
\begin{equation}\label{eq:s}
\mathbf{S}_\alpha[i, j]=\left\{\begin{aligned}
\mathbf{W}_\alpha(i,j)\Omega_{ij}, &\text{\quad if }(i,j)\in\mathbb{E}\\
\mathbf{0}_{q\times q}, &\text{\quad if }(i,j)\notin\mathbb{E}
\end{aligned}\right.
\end{equation}
where $i,j=1,\ldots,n$ and we adopt the convention that $[i,j]$ denotes the $(i,j)$-th block of a block matrix, while $(i,j)$ denotes the $(i,j)$-th entry of a scalar (non-block) matrix. Next, construct a diagonal $n\times n$ block matrix with $q\times q$ entries by
\begin{equation}\label{eq:d}
\mathbf{D}_\alpha[i, i]=\sum_{j=1}^n \mathbf{W}_\alpha(i,j) I_q\,,
\end{equation}
where $i=1,\ldots,n$ and $I_q$ is the identity matrix of size $q\times q$. The corresponding Markov transition matrix is then defined as
\[
\mathbf{M}_\alpha = \mathbf{D}_\alpha^{-1} \mathbf{S}_\alpha\,.
\]
$\mathbf{M}_\alpha$ defines a vector-valued diffusion process on the graph. We call $I_{nq}-\mathbf{M}_\alpha$ the {\em graph connection Laplacian} (GCL). 

Since $\mathbf{M}_\alpha$ is similar to a symmetric metrix $\mathbf{D}_\alpha^{-1/2}\mathbf{S}_\alpha\mathbf{D}^{-1/2}_\alpha$, the eige-decomposition (EVD) of $\mathbf{M}_\alpha$ exists. Denote
\begin{equation*}
\mathbf{M}_\alpha=\mathbf{U}\boldsymbol{\Lambda}\mathbf{U}^{-1}
\end{equation*}
where $\mathbf{U}\in Gl(nq)$ is the matrix of eigenvectors and $\boldsymbol{\Lambda}$ is a diagonal matrix with eigenvalues $\lambda_1 \geq \lambda_2 \geq \cdots \geq \lambda_{nq}$. Denote the corresponding eigenvectors as $u_1, u_2, \ldots, u_{nq}$. Note that by a direct calculation, we have $\lambda_1=1$. For an integer $r \geq 1$ and a diffusion time $t > 0$ chosen by the user, the VDM, denoted as $\Psi_{t,\alpha}$, of the $i$-th data point is defined by
\begin{equation}\label{eq:vdm}
\Psi_{\alpha,t}^{\text{VDM}}:s_i \mapsto \left((\lambda_\ell\lambda_s)^t\langle u_\ell[i],\,u_s[i]\rangle\right)^{r}_{\ell,s=1}\,.
\end{equation}
We refer readers to \cite{singer2011} for more details of GCL and VDM.

\subsection{Proposed LA-VDM}

We now detail the proposed algorithm LA-VDM. Consider a landmark set $\tilde{\mathcal{Z}} = \{a_k\}_{k=1}^m\subseteq \mathbb{R}^p$, which may or may not be a subset of $\tilde{\mathcal{X}}$. A naive choice is uniformly sampling $m$ points from $\tilde{\mathcal{X}}$, or via a designing process to collect $m$ extra points, where $m$ is much smaller than $n$.

We build a bipartite graph connecting $\tilde{\mathcal{X}}$ and $\tilde{\mathcal{Z}}$ by specifying an edge set
$\mathbb{E}\subset\{1,\dots,n\}\times\{1,\dots,m\}$, where $(i,k)\in\mathbb{E}$ indicates that the data point $s_i$ is connected to the landmark $a_k$.
Note that there are various approaches to construct an edge set. A common approach is leveraging the background knowledge. In this work, we assume that we only have the point cloud and a landmark set in the Euclidean space but no further knowledge; that is, we only have the Euclidean metric information.

Next, construct an {\em affinity function} $\bar{w}:\mathbb{E}\to \mathbb{R}_+$. For each edge $(i,k)$, set a positive scalar weight $\bar{w}_{ik} > 0$ 
\[
\bar{w}_{ik}=K\left(\frac{\|s_i-a_k\|_{\mathbb{R}^p}}{\sqrt{\epsilon}}\right)\,,
\]
where $\epsilon>0$ is the bandwidth. We call $\bar{w}_{i k}$ the affinity between $s_i$ and $a_k$.
Furthermore, construct a {\em connection function} $\bar{\Omega}:\mathbb{E}\to O(q)$. For each edge $(i,k)$, we encode a group element 
\[
\bar{\Omega}_{ik} \in O(q)\,, 
\]
where $q\in \mathbb{N}$ and $O(q)$ denotes the orthogonal group in $q$-dimensional space. 
$\bar{\Omega}_{i k}$ is called the connection between $s_i$ and $a_k$, which offers a nonlinear relationship between vector-valued features in $s_i$ and $a_k$. The construction of the connections $\bar{\Omega}_{ij}$ depends on applications, and a detailed formulation is provided in the last of Section \ref{sec:setup} or \eqref{eq:conn}. We call $(\tilde{\mathcal{X}},\tilde{\mathcal{Z}},\mathbb{E},\bar{w},\bar{\Omega})$ the {\em landmark connection graph}.

With the landmark connection graph, construct a \textit{landmark connection} block matrix $\mathbf{S}^{(r)}$ of size $n\times m$ with $q\times q$ entries in the following way. Set
\begin{equation}
\mathbf{S}^{(r)}[i, k]=\left\{
\begin{array}{ll}
\bar{w}_{i k} \bar{\Omega}_{i k} &\text{\quad if }(i,k)\in\mathbb{E}\\
\mathbf{0}_{q\times q} &\text{\quad if }(i,k)\notin\mathbb{E}
\end{array}
\right.
\end{equation}
where the bracket $[i,k]$ denotes the $(i,k)$-th block of the matrix for $i=1,\ldots, n$ and $k=1,\ldots, m$ and $\mathbf{0}_{q\times q}$ is a $q\times q$ matrix with all entries $0$.

Normalization is the key novelty of LA-VDM, which consists of several sub-steps. Construct a \textit{landmark affinity} matrix $\mathcal{W}^{(r)}\in \mathbb{R}^{n\times m}$ by setting
\begin{equation*}
\mathcal{W}^{(r)}(i,k)=\bar{w}_{ik}\,.
\end{equation*}
where the parentheses $(i,k)$ denotes the $(i,k)$-th entry. Construct a diagonal \textit{landmark normalized} matrix $\mathcal{D}_{\innormd}\in \mathbb{R}^{m\times m}$ by setting
\begin{equation*}
\mathcal{D}_{\innormd}(k,k)=e_k^\top (\mathcal{W}^{(r)})^\top\mathcal{W}^{(r)}\mathbf{1}_m\,,
\end{equation*}
where $k=1,\ldots, m$, $e_k\in \mathbb{R}^m$ is a unit vector with the $k$-th entry $1$ and $\mathbf{1}_m$ is a $m\times 1$ vector with all entries $1$.
Then, construct a diagonal block matrix $\mathbf{D}_{\innormd} $ of size $m\times m$ with $q\times q$ entries
\begin{equation}
\mathbf{D}_{\innormd} [k,k]=\mathcal{D}_{\innormd}(k,k)I_q\,,
\end{equation}
where $k=1,\ldots, m$ and $I_q$ is the $q\times q$ identity matrix. For a user-defined hyperparameter $\beta\in [0, 1]$, the \textit{landmark $\beta$-normalized affinity} matrix is defined as
\begin{equation}
\mathcal{W}_{\beta}:=\mathcal{W}^{(r)}\mathcal{D}_{\innormd}^{-\beta}(\mathcal{W}^{(r)})^{\top}\in \mathbb{R}^{n\times n}\,.
\end{equation}
Note that $\mathcal{D}_{\innormd}^{-\beta}$ plays the role of normalizing the impact of non-uniform sampling density of the landmark set, so the $(i,j)$-th entry is an affinity between $s_i$ and $s_j$ less dependent on the landmark sampling scheme. 
Construct a \textit{dataset normalized} diagonal matrix $\mathcal{D}_{\exnormd,\beta}\in \mathbb{R}^{n\times n}$ by setting
\begin{equation*}
\mathcal{D}_{\exnormd,\beta}(i,i) :=\sum_{j=1}^n\mathcal{W}_\beta(i,j)\,,
\end{equation*}
and a diagonal block matrix $\mathbf{D}_{\exnormd,\beta}$ of size $n\times n$ with $q\times q$ entries
\begin{equation}
\mathbf{D}_{\exnormd,\beta} [i,i]=\mathcal{D}_{\exnormd}(i,i)I_q\,,
\end{equation}
where $i=1,\ldots, n$. For a user-defined hyperparameter $\alpha\in [0, 1]$, define the $(\beta,\alpha)$-\textit{normalized affinity} matrix as 
\begin{equation*}
\mathcal{W}_{\beta,\alpha}:=\mathcal{D}_{\exnormd,\beta}^{-\alpha}\mathcal{W}_{\beta}\mathcal{D}_{\exnormd,\beta}^{-\alpha}\in \mathbb{R}^{n\times n}\,.
\end{equation*}
This normalization is designed to adjust the influence of the sampling density of the point cloud $\mathcal{X}$, which mimics the $\alpha$-normalized affinity matrix in \cite{coifman2006,singer2011,singer2015}. Finally, define a diagonal block degree matrix $\mathbf{D}_{\beta,\alpha}$ of size $n\times n$ with $q\times q$ entries by setting 
\begin{equation}\label{eq:deg}
\mathbf{D}_{\beta,\alpha} [i,i]=\sum_{j=1}^n\mathcal{W}_{\beta,\alpha}(i,j)I_q\,,
\end{equation}
where $i=1,\ldots, n$.

With $\mathbf{D}_{\beta,\alpha}$, $\mathbf{S}^{(r)}$, $\mathbf{D}_{\exnormd,\beta}$ and $\mathbf{D}_{\innormd}$, apply singular value decomposition (SVD) to $\mathbf{D}_{\beta,\alpha}^{-1/2}\mathbf{D}_{\exnormd,\beta}^{-\alpha}\mathbf{S}^{(r)}\mathbf{D}_{\innormd}^{-\beta/2}\in \mathbb{R}^{nq\times mq}$ and obtain:
\begin{equation}\label{eq:lan_evec}
\mathbf{D}_{\beta,\alpha}^{-1/2}\mathbf{D}_{\exnormd,\beta}^{-\alpha}\mathbf{S}^{(r)}\mathbf{D}_{\innormd}^{-\beta/2}=\bar{\mathbf{U}} \bar{\mathbf{\Lambda}}\bar{\mathbf{V}}^\top\in\mathbb{R}^{nq\times mq}\,.
\end{equation}
where $\bar{\mathbf{U}}\in O(nq)$, $\bar{\mathbf{V}}\in O(mq)$, and the diagonal entries of $\bar{\mathbf{\Lambda}}$ are $\bar{\sigma}_1 \geq \bar{\sigma}_2 \geq \cdots \geq \bar{\sigma}_{mq}\geq 0$. Set 
\[
\bar{\mathbf{U}}:=\mathbf{D}_{\beta,\alpha}^{-1/2}\bar{\mathbf{U}}\ \mbox{ and }\ \bar{\mathbf{\Lambda}}:=\bar{\mathbf{\Lambda}}^{2}\,. 
\]
Choose $1 \leq r \leq m$. 
Denote $\bar{\mathbf{u}}_\ell\in \mathbb{R}^{nq}$ to be $\ell$-th column of $\bar{\mathbf{U}}_r$. The {\em $(\beta,\alpha)$-landmark vector diffusion map} ($(\beta,\alpha)$-LA-VDM) $\bar{\Psi}_t:\tilde{\mathcal{X}}\to\mathbb{R}^{r\times r}$ is defined as
\begin{equation}\label{eq:lvdm}
\bar{\Psi}_{\beta,\alpha,t}:s_i \mapsto \left((\bar{\sigma}^2_\ell\bar{\sigma}^2_s)^t\langle\bar{\mathbf{u}}_\ell[i],\,\bar{\mathbf{u}}_s[i]\rangle\right)^{r}_{\ell,s=1}
\end{equation}
where $t>0$ is the diffusion time chosen by the user.

\subsection{Computational complexity}
The computational complexity of LA-VDM consists of four major components. We begin by analyzing the construction of $\mathbf{D}_{\mathcal{Z}}$, which can be divided into two primary steps. First, computing $\mathcal{W}^{(r)}\mathbf{1}_m \in \mathbb{R}^n$ requires $O(nm)$ operations. Then, substituting this result into the expression $e_k^\top (\mathcal{W}^{(r)})^\top \mathcal{W}^{(r)}\mathbf{1}_m$ also takes $O(nm)$ time. 
All other quantities can be evaluated with similar computational effort, noting that any matrix inversions involved are only on diagonal matrices and hence can be computed in linear time. Consequently, the total cost to construct $\mathbf{D}_{\mathcal{Z}}$, $\mathbf{D}_{\mathcal{X},\beta}$, $\mathbf{D}_{\beta,\alpha}$, and the normalized form $\mathbf{D}_{\beta,\alpha}^{-1/2}\mathbf{D}_{\mathcal{X},\beta}^{-\alpha} \mathbf{S}^{(r)} \mathbf{D}_{\mathcal{Z}}^{-\beta/2}$ is $O(nm)$.
The dominant computational cost arises from the SVD of the matrix 
$\mathbf{D}_{\beta,\alpha}^{-1/2} \mathbf{D}_{\mathcal{X},\beta}^{-\alpha} \mathbf{S}^{(r)} \mathbf{D}_{\mathcal{Z}}^{-\beta/2}$,
which takes $O(nm^2)$ time under the assumption that $m < n$.
Therefore, the overall computational complexity of LA-VDM is $O(nm^2)$. Notably, if $m < n^{1/2}$, LA-VDM offers improved efficiency compared to the original VDM, which typically requires $O(n^{2.81})$ operations, or $O(n^{2+\epsilon})$ under sparse graph assumptions.

\subsection{Relation between vanilla vector diffusion maps}\label{sec:relation}

Given a landmark connection matrix $\mathbf{S}^{(r)}$, we define
\begin{align*}
\mathbf{S}_{\beta,\alpha}&:=\mathbf{D}_{\exnormd,\beta}^{-\alpha}\mathbf{S}^{(r)}\mathbf{D}_{\innormd}^{-\beta/2}(\mathbf{D}_{\exnormd,\beta}^{-\alpha}\mathbf{S}^{(r)}\mathbf{D}_{\innormd}^{-\beta/2})^\top\\
&~=\mathbf{D}_{\exnormd,\beta}^{-\alpha}\mathbf{S}^{(r)}\mathbf{D}_{\innormd}^{-\beta}\mathbf{S}^{(r)\top}\mathbf{D}_{\exnormd,\beta}^{-\alpha}\in\mathbb{R}^{nq\times nq}\,,
\end{align*}
which encodes both affinity and connection information of the point cloud $\tilde{\mathcal{X}}$. Since the multiplication of non-negative weights remains non-negative and the product of orthogonal matrices in $O(d)$ remains orthogonal, the matrix $\mathbf{S}_{\beta,\alpha}$ plays a role analogous to $\mathbf{S}$ in the vanilla VDM \eqref{eq:s}. For simplicity, we set $\alpha=\beta=0$. Under this setting, the affinity-connection between two data points $s_i$ and $s_j$ via the landmark set $\tilde{\mathcal{Z}}$ is given by
\[
\mathbf{S}_{\beta,\alpha}(i,j) = \sum_{z_k \in \mathcal{Z}} \bar{w}_{ik} \bar{w}_{jk} \bar{\Omega}_{ik} \bar{\Omega}_{jk}^\top.
\]
Accordingly, the matrix $\mathbf{D}_{\beta,\alpha}$, defined in~\eqref{eq:deg}, serves as the degree matrix associated with $\mathbf{S}_{\beta,\alpha}$, analogous to $\mathbf{D}$ in the vanilla VDM construction \eqref{eq:d}. If we define the transition matrix as
\begin{equation} \label{eq:app_markov}
\mathbf{M}_{\beta,\alpha} :=  \mathbf{D}_{\beta,\alpha}^{-1} \mathbf{S}_{\beta,\alpha},
\end{equation}
we immediately see that it characterizes the diffusion of vector-valued information between $s_i$ and $s_j$ via the landmark set.

To understand the connection between this construction and the singular vectors used in~\eqref{eq:lan_evec}, we observe the identity
\begin{align*}
&\mathbf{D}_{\beta,\alpha}^{-1/2}\mathbf{D}_{\exnormd}^{-\alpha}\mathbf{S}^{(r)}\mathbf{D}_{\innormd}^{-\beta/2}\left[\mathbf{D}_{\beta,\alpha}^{-1/2}\mathbf{D}_{\exnormd}^{-\alpha}\mathbf{S}^{(r)}\mathbf{D}_{\innormd}^{-\beta/2}\right]^\top=\mathbf{D}_{\beta,\alpha}^{-1/2}\mathbf{S}_{\beta,\alpha}\mathbf{D}_{\beta,\alpha}^{-1/2}\,,
\end{align*}
which implies that the left singular vectors of $\mathbf{D}_{\beta,\alpha}^{-1/2}\mathbf{D}_{\exnormd}^{-\alpha}\mathbf{S}^{(r)}\mathbf{D}_{\innormd}^{-\beta/2}$
correspond directly to the eigenvectors of the transition matrix $\mathbf{M}_{\beta,\alpha}$.

The parameters $\beta$ and $\alpha$ serve distinct normalization purposes. The parameter $\beta$ acts as an intrinsic normalizer, mitigating the influence of varying densities of landmark $\tilde{\mathcal{Z}}$ on the embedding. This normalization is similar to that in \cite{yeh2024landmark}, while in this work all data points are assumed to lie on the same space. In contrast, $\alpha$ serves as an extrinsic normalizer that compensates for non-uniform sampling density in the original dataset $\tilde{\mathcal{X}}$. This is analogous to the symmetric normalization in the classical VDM algorithm, where the affinity-connection matrix $\mathbf{S}$ is replaced by $\mathbf{D}^{-\alpha} \mathbf{S} \mathbf{D}^{-\alpha}$ (cf. Equation (4.8) in~\cite{singer2011} or Equation (19) in~\cite{singer2015}), with $\mathbf{S}$ and $\mathbf{D}$ defined in~\eqref{eq:s} and~\eqref{eq:d}.

In the next section, we show that under mild assumptions, the affinity-connection matrix $\mathbf{S}$ from the original VDM can be well approximated by $\mathbf{S}_{\beta,\alpha}$ constructed via landmarks. This leads to an accurate and efficient approximation of the original VDM embeddings, but with significantly reduced computational complexity.

\section{Asymptotic Behavior of LA-VDM under the Manifold Setup}\label{sec:thm}

In this section, we establish theoretical guarantees for the LA-VDM algorithm introduced in Section~\ref{sec:proposed}, within the setting of principal bundles. Beyond accelerating vanilla VDM, we show that the proposed two-stage normalization, the $\alpha$- and $\beta$-normalizers, effectively address the sampling density issue that remains unsolved in ROSELAND. A key challenge arises from the use of parallel transport, where transporting a vector from one to another point can yield different results depending on the chosen path, due to manifold curvature. This path dependence complicates the analysis of LA-VDM. We provide a bound on the discrepancy introduced by such two-step landmark-based transport, and show that the resulting error remains controlled within an $\epsilon$-neighborhood. We refer readers to Section \ref{section: necessary diff geo background} for a quick review of necessary differential geometry background.

\subsection{Mathematical model and sampling scheme}\label{sec:setup}

Consider a compact, smooth, $d$-dimensional Riemannian manifold $(\mathcal{M}, g)$ without boundary, which is embedded isometrically within $\mathbb{R}^p$ through a map $\iota: \mathcal{M} \hookrightarrow \mathbb{R}^p$.
Then, consider a principal bundle $P(\mathcal{M}, O(q))$ over $\M$, where $q\in \mathbb{N}$ and the structure group is the orthogonal group $O(q)$. This bundle is associated with a vector bundle $\mathcal{E}=P(\mathcal{M}, O(q))\times_{O(q)}\mathbb{R}^q$. Note that in general $q\neq d$. The tangent bundle $T\M$ and its associated frame bundle $O(\M)$ is a special case, where $q=d$.

Consider two independent random vectors, $\tilde{X} = \iota \circ X$ and $\tilde{Z} = \iota \circ Z$, where $X, Z: (\Omega, \mathcal{F}, P) \to \mathcal{M}$ are $\mathcal{M}$-valued random variables, that induces a probability measures $\nu := X_*P$ and $\nu_{\mathcal{Z}} := Z_*P$ on $\mathcal{M}$ respectively.
Assume $\nu$ and $\nu_{\mathcal{Z}}$ are both absolutely continuous with respect to the Rimennanian density $dV$, with smooth probability density functions $p$ and $p_{\mathcal{Z}}$ on $\M$ such that $d\nu = p dV$ and $d\nu_{\mathcal{Z}} = p_{\mathcal{Z}} dV$ respectively. We summarize these assumptions here.

\begin{assumption}\label{asp:1}
\begin{enumerate}
\item The $d$-dimensional manifold $(\mathcal{M},g)$ is smooth and compact without boundary and isometrically embedded into $\mathbb{R}^{p}$ via $\iota$.
\item Fix a principal bundle $P(\mathcal{M}, O(q))$, which is associated with a vector bundle $\mathcal{E}:=P(\mathcal{M}, O(q))\times_{O(q)} \mathbb{R}^q$ to be the associated vector bundle with a metric $g^{\mathcal{E}}$ and the metric connection $\nabla^{\mathcal{E}}$.

\item $\nu$ and $\nu_{\mathcal{Z}}$ are absolutely continuous with respect to $dV$. Moreover, $p,p_{\mathcal{Z}}\in C^{4}(\M)$ and are both bounded below from zero.
\end{enumerate}
\end{assumption}

We model the point cloud $\mathcal{X} = \{x_i\}_{i=1}^n$ to be sampled independent and identically distributed (i.i.d.) from $X$ and model the landmark set $\mathcal{Z} = \{z_k\}_{k=1}^m$ to be i.i.d. sampled from $Z$. 
For each data point $x_i$, we randomly assign $u_i \in P(\mathcal{M}, O(q))$ such that $\pi(u_i) = x_i$, and similarly for the landmark point. To simplify the notation, we write $u_i := u_{x_i}$.

Let $V_{\mathcal{X}} := \oplus_{x_i \in \mathcal{X}} \mathbb{R}^q$ and $\mathcal{E}_{\mathcal{X}} := \oplus_{x_i \in \mathcal{X}} \mathcal{E}_{x_i}$ be the $n q$-dimensional Euclidean vector space and the discretized vector bundle, respectively. For $\mathbf{v}\in V_{\mathcal{X}}$, we use $\mathbf{v}[\ell]$ to denote its $\ell$-th component in the direct sum for $\ell = 1, \ldots, n$. Similarly, For $\mathbf{w}\in \mathcal{E}_{\mathcal{X}}$, we use $\mathbf{w}[\ell]$ to denote its $\ell$-th component in the direct sum for $\ell = 1, \ldots, n$. By the construction, there exist invertible linear maps from block entries of $V_{\mathcal{X}}$ to fibers $\mathcal{E}_{\mathcal{X}}$. 
Indeed, we have $B_{\mathcal{X}}: V_{\mathcal{X}} \rightarrow \mathcal{E}_{\mathcal{X}}$ given by
\begin{equation*}
B_{\mathcal{X}} \mathbf{v}:=\left[u_{1} \mathbf{v}[1], \ldots, u_{n} \mathbf{v}[n]\right] \in \mathcal{E}_{\mathcal{X}}, 
\end{equation*}
where $\mathbf{v} \in V_{\mathcal{X}}$.
Denote $B_{\mathcal{X}}^{-1}: \mathcal{E}_{\mathcal{X}} \rightarrow V_{\mathcal{X}}$ to be the inverse of $B_{\mathcal{X}}$ satisfying
\begin{equation*}
B_{\mathcal{X}}^{-1} \mathbf{w}:=\left[u_{1}^{-1} \mathbf{w}[1], \ldots u_{n}^{-1} \mathbf{w}[n]\right] \in V_{\mathcal{X}},
\end{equation*}
where $u_i^{-1}$ is the inverse of $u_i$ and $\mathbf{w} \in \mathcal{E}_{\mathcal{X}}$. Note that $B_{\mathcal{X}}^{-1} B_{\mathcal{X}} \mathbf{v}=\mathbf{v}$ for all $\mathbf{v} \in V_{\mathcal{X}}$.

Next, define $\delta_{\mathcal{X}}: X \in C(\mathcal{E}) \rightarrow \mathcal{E}_{\mathcal{X}}$ by
\begin{equation*}
\delta_{\mathcal{X}} X:=\left[X\left(x_{1}\right), \ldots X\left(x_{n}\right)\right] \in \mathcal{E}_{\mathcal{X}}\,,
\end{equation*}
which is a discretization operator that samples the section $X$ at $\mathcal{X}$. Then, for $X\in C(\mathcal{E})$, 
\begin{equation}\label{eq:disc_amb_vf}
\boldsymbol{X}:=B_{\mathcal{X}}^{-1} \delta_{\mathcal{X}} X \in V_{\mathcal{X}}
\end{equation}
represents the discretization of the vector field $X$ evaluated at the sample points $\left\{x_i\right\}_{i=1}^n$ and expressed in the coordinates associated with bases for $\left\{\mathcal{E}_{x_i}\right\}_{i=1}^n$. Note that $\boldsymbol{X}$ is an $n q$-dimensional vector consisting of $n$ blocks, each of which of dimension $q$. While $\delta_{\mathcal{X}}X$ is an abstract object, $\boldsymbol{X}$ can be saved in computer for further analysis.

With the point cloud $\mathcal{X}$ and landmark $\mathcal{Z}$, we can apply LA-VDM once we decide how to determine affinity and connection information. Given a kernel $K: \mathbb{R}_{\geq 0}\to \mathbb{R}_{\geq 0}$, the affinity between $x_i$ and $z_j$ is defined as
\begin{equation}\label{eq:kernel}
K_\epsilon(x_i,z_j) = K\left( \frac{ \| \iota(x_i) - \iota(z_j) \| }{ \sqrt{\epsilon} } \right),
\end{equation}
where $\epsilon > 0$ is a scale hyperparameter. We impose the following assumption for the kernel function.

\begin{assumption}\label{asp:2}
The kernel $K$ is $C^3$, positive and decay exponentially fast; that is, exist $c_1,c_2>0$ such that
\begin{equation*}
K(t)<c_1e^{-c_2t^2} \quad\text{and}\quad |\partial_t K(t)|<c_1e^{-c_2t^2}\,.
\end{equation*}
Define $\mu_{r,k,\ell}:=\int_{\mathbb{R}^d}\|x\|^r \partial^{(k)} \left(K(\|x\|)\right)^\ell d x$. Assume that $\mu_{0,0,1}=1$.
\end{assumption}

The connection is tricky. In the case we only have point cloud, we need a way to access the connection once we determine the principal bundle. The selected principal bundle, and hence the connection, depends on the application, so there is no universal way to estimate the connection. 
For example, if the goal is to obtain the connection Laplacian associated with the tangent bundle, we take the frame bundle as the principal bundle, and estimate the parallel transport by combining local principal component analysis (PCA) and rotational alignment like that in the original VDM work \cite{singer2011}. If the interest is estimating the orientation of the manifold, then the principal bundle is $\mathbb{Z}_2$-bundle and the connection can be estimated by local PCA and sign matching.
In this work, since the focus is landmark diffusion and speedup, to avoid distraction and simplify the discussion, we assume that we can access or accurately estimate the connection information. In other words, we assume that we can access the parallel transport $\parallelslant_{y}^{x}$ between any pair of close points $x,y\in \M$. This mild assumption can be easily removed in many practical scenarios. 
Under this assumption, the symmetric connection group $\Omega_{ij}$ is then defined by
\begin{equation}\label{eq:conn}
\Omega_{ij} := u_i^{-1} \parallelslant_{x_j}^{x_i} u_j,
\end{equation}
where $\parallelslant_{x_j}^{x_i}$ denotes parallel transport along the geodesic connecting $x_j$ to $x_i$, and $u_i, u_j \in P(\M, O(q))$.

\subsection{``Double'' Parallel Transport via landmark}
A central challenge in analyzing the LA-VDM framework lies in understanding how the landmark constraint influences the estimation of parallel transport. In the original VDM framework \cite{singer2011}, it was shown that, under appropriate construction, the parallel transport between two nearby points can be accurately estimated from the point cloud alone, without requiring additional geometric information. 
In contrast, LA-VDM introduces a detour: even if two points are close in the ambient space or intrinsic geometry, the landmark-based scheme forces transport to occur indirectly; that is, first from one point to a (possibly distant) landmark, and then back to the neighboring point. This deviation from the geodesic path introduces a significant complication. It recalls the well-known fact that parallel transport of a vector along different paths between the same endpoints may produce different results, due to the curvature of the underlying manifold. This naturally raises concerns about the accuracy of parallel transport approximations between nearby points in the presence of landmark constraints, even if we can accurately estimate the parallel transport between any two points.
Figure \ref{fig:parallel estimation LAVA} illustrates this issue schematically. A concrete example in $S^2$ is provided in Example \ref{Example: parallel transport S^2} in the Supplementary. The trade-off between computational efficiency and geometric fidelity must therefore be carefully considered in the design and application of LA-VDM.

\begin{figure}[htb!]
\begin{minipage}{0.48\textwidth}
\centering
\includegraphics[trim=0 0 20 0, clip, width=1\textwidth]{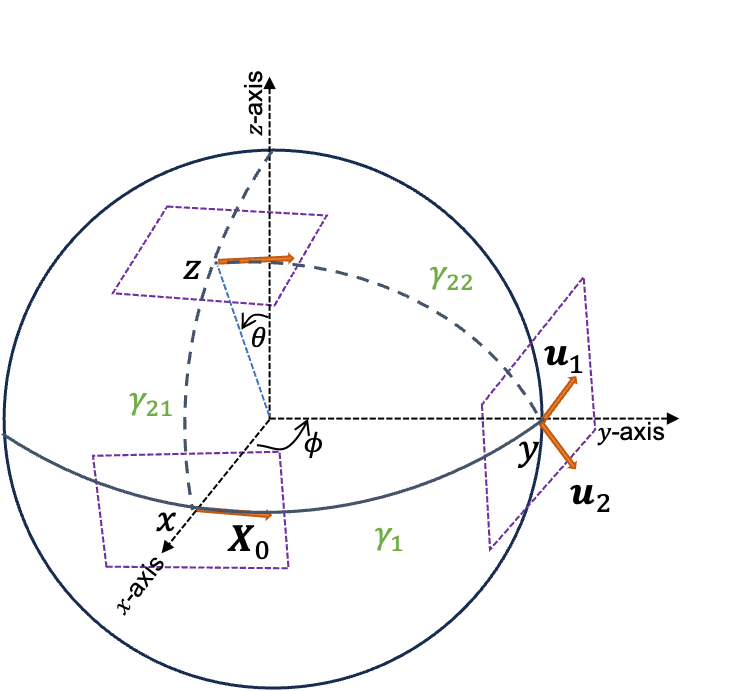} 
\end{minipage}
\begin{minipage}{0.5\textwidth}
\caption{Illustration of a potential issue in estimating parallel transport from $x$ to 
$y$ via landmark diffusion: different paths yield different parallel-transported vectors at $y$.}
\label{fig:parallel estimation LAVA}
\end{minipage}
\end{figure}

To quantify the effect of this procedure and show that we can still accurately estimate the parallel transport when two points are close, we need to estimate the difference between the vector obtained through this accelerated landmark diffusion scheme and the one obtained via a direct single-step transport using the vanilla VDM method. This estimation is provided by the following lemma, which is based on \cite[Equation (3)]{gavrilov2007double}. A full proof is postponted to Section \ref{sec:proof}.
\begin{lemma}\label{lem:double_para}
Consider a vector bundle $X\in C^3(\E)$. Let $x,y,z\in \M$ be mutually contained within each other's $\sqrt{\epsilon}$-neighborhoods. Then, we have
\begin{equation*}
\parallelslant_z^x\parallelslant_y^z X(y)=\parallelslant_y^x X(y) + \mathcal{O}(\epsilon^{3/2})
\end{equation*}
where the implied constant depends on the curvature and $C^3$ norm of $X$.
\end{lemma}

\subsection{Pointwise convergence}\label{sec:main_thm}

With the above setup, we adopt the standard analytical framework to study the asymptotic behavior of the proposed LA-VDM algorithm. 
We follow the principal bundle framework and the sampling scheme defined in Section \ref{sec:setup}. We start with introducing some quantities.

\begin{definition}\label{def:int}
Define the landmark normalized function, denoted as $d_{\innormd, \epsilon}: \mathcal{M} \rightarrow \mathbb{R}_{+}$, by

\begin{equation*}
d_{\innormd, \epsilon}(z)=\int_{\mathcal{M}}\int_{\mathcal{M}} K_\epsilon(z,x)K_\epsilon(x,z') d\nu(x)d \nu_{\mathcal{Z}}\left(z^{\prime}\right)\,.
\end{equation*}
\end{definition}

\begin{definition}\label{def:ext}
Take $\beta\geq 0$. Define the landmark $\beta$-normalized kernel function $K_{\epsilon,\beta}: \mathcal{M} \times \mathcal{M} \rightarrow \mathbb{R}_{+}$ by
\begin{equation*}
K_{\epsilon,\beta}\left(x,y\right):=\int_{\mathcal{M}} K_\epsilon(x,z) K_\epsilon\left(z,y\right)d^{-\beta}_{\innormd}(z) d \nu^{\Z}(z)\,.
\end{equation*}
Define the normalized function $d_{\exnormd, \epsilon,\beta}: \mathcal{M} \rightarrow \mathbb{R}_{+}$ by
\begin{equation*}
d_{\exnormd, \epsilon,\beta}(x)=\int_{\mathcal{M}} K_{\epsilon,\beta}\left(x,y\right) d \nu\left(y\right)\,.
\end{equation*}
Define the landmark $\beta$-normalized transport operator $S(x,y):E_y\to E_x$ by
\begin{equation}\label{eq:transp_def}
S_{\epsilon,\beta}\left(x,y\right)u:=\int_{\mathcal{M}} K_\epsilon(x,z) K_\epsilon\left(z,y\right)d^{-\beta}_{\innormd}(z)\parallelslant^x_z\parallelslant^z_y ud \nu^{\Z}(z)\,,
\end{equation}
where $u\in E_y$.
\end{definition}

Next, define the landmark $(\beta,\alpha)$-\textit{normalized} kernel function and the landmark $(\beta,\alpha)$-\textit{normalized} transport operator.

\begin{definition}
Take $\alpha,\beta\geq 0$. Define the landmark $(\beta,\alpha)$-normalized kernel function $K_{\epsilon,\beta,\alpha}:\M\times \M\rightarrow \mathbb{R}_+$ by
\begin{equation*}
K_{\epsilon,\beta,\alpha}\left(x,y\right):=\frac{K_{\epsilon,\beta}(x,y)}{d^{\alpha}_{\exnormd,\epsilon,\beta}(x)d^{\alpha}_{\exnormd,\epsilon,\beta}(y)}\,,
\end{equation*}
and the corresponding degree function $d_{\epsilon,\beta,\alpha}$ by
\begin{equation*}
d_{\epsilon,\beta,\alpha}(x):=\int_{\M} K_{\epsilon,\beta,\alpha}\left(x,y\right)d\nu(y)\,.
\end{equation*}
Define the landmark $(\beta,\alpha)$-normalized transport operator $S_{\epsilon,\beta,\alpha}(x,y):E_y\to E_x$ by
\begin{equation*}
S_{\epsilon,\beta,\alpha}\left(x,y\right):=\frac{S_{\epsilon,\beta}(x,y)}{d^{\alpha}_{\exnormd,\epsilon,\beta}(x)d^{\alpha}_{\exnormd,\epsilon,\beta}(y)}\,,
\end{equation*}
\end{definition}

\begin{definition}
(Landmark $(\beta,\alpha)$-normalized Vector Diffusion Operator) The landmark vector diffusion operator, $T_{\epsilon,\beta,\alpha}: C(\mathcal{E}) \rightarrow C(\mathcal{E})$, is defined as
\begin{equation*}
T_{\epsilon,\beta,\alpha}X(x)=\int_{\mathcal{M}} \frac{S_{\epsilon,\beta,\alpha}(x, y)}{d_{\epsilon,\beta,\alpha}(x)}  X(y) d \nu(y)\,,
\end{equation*}
where $X\in C(\mathcal{E})$.
\end{definition}

Note that the operator $S_{\epsilon,\beta,\alpha}$ does not directly parallel transport $X(y)$ to $x$. Instead, it integrates the values of $X(y)$ that have been parallel transported to $x$ along several two-stage piecewise geodesics, each beginning at $y$ and ending at $x$, with every path passing through a landmark point $z$. 

The asymptotic behavior of the LA-VDM algorithm is analyzed by separating the bias and variance contributions. In the variance analysis, as $n \to \infty$, we characterize how the finite sample size influences the convergence of $\mathbf{M}_{\beta,\alpha}$ to the landmark diffusion integral operator $T_{\epsilon,\beta,\alpha}$. In the bias analysis, as $\epsilon \to 0$, we demonstrate that $T_{\epsilon,\beta,\alpha}$ converges to a perturbed connection Laplacian operator that depends on the proposed two-stage density normalization. Combining these results yields the overall convergence rate.

\subsubsection{Effective landmark kernel}
We could view $S_{\epsilon,\beta}$ or $S_{\epsilon,\beta,\alpha}$ as an {\em effective} transport operator associated with LA-VDM in the sense that it describes the overall behavior of the landmark-constrained dynamics. Under this perspective, we could discuss the associated {\em effective landmark kernel}. We have the following property.

\begin{proposition}\label{prop:eff}
Suppose $K$ is a Gaussian kernel. For $u\in E_y$, when $\epsilon>0$ is sufficiently small, we have 
\begin{align*}
 S_{\epsilon,\beta,\alpha}(x,y)u
=\epsilon^{-d\alpha+d\beta\alpha-d\beta+d/2} c_{\beta,\alpha}(x,y)\exp\!\left(-\frac{\|x-y\|^2}{2\epsilon}\right)\parallelslant^x_y u
+\mathcal{O}(\epsilon^{3/2}),   
\end{align*}
where $c_{\beta,\alpha}$ depends on $x$ and $y$ and is of order 1 since $\M$ is compact.
The effective landmark kernel is thus defined as 
\[
K_{\epsilon,\beta,\alpha}(x,y):=\epsilon^{-d\alpha+d\beta\alpha-d\beta+d/2}c_{\beta,\alpha}(x,y)\exp\!\left(-\frac{\|x-y\|^2}{2\epsilon}\right),
\]
which is no longer isotropic but depends on $x$ and $y$.
\end{proposition}

%
This result is anticipated since within a sufficiently small neighborhood, a two-stage transport can be well-approximated by a single transport, but the weight depends on the landmark. 
Note that when the vector bundle is a trivial line bundle, this effective kernel is reduced to the effective kernel in ROSELAND considered in Equations~(20, 46) of~\cite{shen2022}.
A numerical visualization of this result can be found later in Section \ref{section: effective kernel numerics}.

\subsubsection{Bias analysis}
Based on Lemma \ref{lem:double_para} for double parallel transport, we establish the relation between vector diffusion integral operator $T_{\epsilon,\beta,\alpha}$ and the connection Laplacian operator.
\begin{theorem}\label{thm:bias}
Assume Assumptions \ref{asp:1} and \ref{asp:2} hold.
Fix $x\in\M$. Let $\{E_i\}_{i=1}^d$ be an orthonormal basis of $T_x \mathcal{M}$ associated with the manifold metric $g$. For $X \in C^3(\mathcal{E})$, when $\epsilon>0$ is sufficiently small, we have
\begin{align*}
T_{\epsilon,\beta,\alpha} X(x)=\,&X(x)+\frac{\epsilon \mu_{2,0,1}}{d} \nabla^{2} X(x)\\
&+\frac{\epsilon \mu_{2,0,1}}{d}\sum^{d}_{i=1}\left(\frac{2 \nabla_{E_i} \rho_{\beta,\alpha}(x)}{\rho_{\beta,\alpha}(x)}+\frac{\nabla_{E_i} q_\beta(x)}{q_\beta(x)}\right) \nabla^{\mathcal{E}}_{E_i} X(x)+\mathcal{O}\left(\epsilon^{3/2}\right)\,,
\end{align*}
where $q_\beta(x)=p(x)^{-\beta}p_{\Z}(x)^{1-\beta}$, $\rho_{\beta,\alpha}(x)=q_\beta(x)^{-\alpha}p(x)^{1-\alpha}$ and the implied constant depends on the scalar curvature at $x$, second fundamental form
at $x$, $C^3$ norm of $X$ and $C^2$ norms of $p$ and $p_{\Z}$.
\end{theorem}

\begin{remark}
Note that when $\alpha=\beta=0$ and the bundle is the trivial line bundle, LA-VDM reduced to ROSELAND \cite{shen2022}, and Theorem \ref{thm:bias} specializes to the corresponding Theorem 1 in \cite{shen2022}. 
\end{remark}

We have several immediate corollaries.

\begin{corollary}\label{cor:b12}
Grant the notation and assumption in Theorem \ref{thm:bias}.
Further suppose $p_\Z=p$ and $\beta=1/2$. We have
\begin{align*}
T_{\epsilon,1/2,\alpha} X(x)=\,&X(x)+\frac{\epsilon \mu_{2,0,1}}{d} \nabla^{2} X(x)\\
&+2\frac{\epsilon \mu_{2,0,1}}{d}\sum^{d}_{i=1}\frac{ \nabla_{E_i} p(x)^{1-\alpha}}{p(x)^{1-\alpha}} \nabla^{\mathcal{E}}_{E_i}X(x)+\mathcal{O}\left(\epsilon^{3/2}\right)\,.
\end{align*}
If we further assume $p$ is constant, then 
\begin{align*}
T_{\epsilon,1/2,\alpha}X(x)=X(x)+\frac{\epsilon \mu_{2,0,1}}{d}\nabla^2 X(x)+\mathcal{O}(\epsilon^{3/2})\,.
\end{align*}
\end{corollary}

Landmark design is critical for the LA-VDM performance. This corollary characterizes the behavior of LA-VDM under the naive but powerful landmark design, where a subset of the dataset is uniformly selected as the landmark set so that we have $p_{\mathcal{Z}}=p$.
Corollary \ref{cor:b12} states that with this landmark design, setting $\beta=1/2$ normalizes the landmark sampling distribution $p_\Z$, rendering it inconsequential to LA-VDM. 
Note that the constant of the leading term in the asymptotic operator $\frac{1-T_{\epsilon,1/2,0}}{\epsilon}$ is twice that of the vanilla VDM. This factor of two arises because the landmark-constrained diffusion effectively entails a two-step diffusion process. Note that when $\alpha=0$, both LA-VDM and the vanilla VDM depends on $p$.

\begin{corollary}\label{cor:b1}
Grant the notation and assumption in Theorem \ref{thm:bias}.
If we take $\beta=1$, then $q_1(x)=\frac{1}{p(x)}$, which is independent on $p_\Z$. We have
\begin{align*}
T_{\epsilon,1,\alpha} X(x)=\,&X(x)+\frac{\epsilon \mu_{2,0,1}}{d} \nabla^{2} X(x)\\
&+\frac{\epsilon \mu_{2,0,1}}{d}\sum^{d}_{i=1}\frac{ \nabla_{E_i} p(x)}{p(x)} \nabla^{\mathcal{E}}_{E_i} X(x)+\mathcal{O}\left(\epsilon^{3/2}\right)\,.
\end{align*}
If we further assume $p$ is constant; that is, the sampling is uniform, then we obtain
\begin{align*}
T_{\epsilon,1,\alpha}X(x)=X(x)+\frac{\epsilon \mu_{2,0,1}}{d}\nabla^2 X(x)+\mathcal{O}(\epsilon^{3/2})\,.
\end{align*}
\end{corollary}

Corollaries \ref{cor:b12} and \ref{cor:b1} show that the normalization factor $\beta$ plays a role analogous to the $\alpha$-normalization in the vanilla VDM or DM algorithms \cite{coifman2006,singer2011,singer2015}, but with the key difference that it is intended to eliminate the influence of the landmark distribution rather than that of the original dataset distribution. In Corollary \ref{cor:b12}, this is achieved by appropriately choosing $p_\Z$ to cancel out the landmark distribution's effect and recover the results of vanilla VDM. In contrast, Corollary \ref{cor:b1} sets $\beta = 1$, making the choice of $p_\Z$ irrelevant to the original dataset distribution. Although under this condition ($\beta=1$) the asymptotic operator becomes independent of the landmark set, it generally differs from the operator utilized in the vanilla VDM.

After eliminating the influence of the landmarks via $\beta$ and the choice of $p_\Z$, one can follow the approach in \cite{coifman2006,singer2011,singer2015} and use $\alpha$ to adjust for the effect of the original dataset's sampling density, yielding embeddings that are independent of $p$ when $\alpha=1$. The result is summarized in the following Corollary. 

\begin{corollary}\label{cor:a1}
Grant the notation and assumption in Theorem \ref{thm:bias}.
If $p_\Z=p$ and we take $\alpha=1$ and $\beta=1/2$, then we have
\begin{align*}
T_{\epsilon,1/2,1} X(x)=X(x)&+\frac{\epsilon \mu_{2,0,1}}{d} \nabla^{2} X(x)+\mathcal{O}\left(\epsilon^{3/2}\right)\,.
\end{align*}
\end{corollary}

Note that in the setup of Corollary \ref{cor:a1}, $q_{1/2}=\rho_{1/2,1}=1$, so the asymptotic operator is independent of both $p_\Z$ and $p$. This result indicates that, when the landmarks are uniformly sampled from the dataset, LA-VDM recovers the intrinsic connection Laplacian operator if we choose $\alpha=1$ and $\beta=1/2$.

\subsubsection{Variance Analysis}

As $n\to\infty$, we show that the matrix $I_n-\left(\mathbf{D}_{\beta,\alpha}\right)^{-1} \mathbf{S}_{\beta,\alpha}$ used in LA-VDM asymptotically behaves like the integral operator $ 1-T_{\epsilon,\beta,\alpha} $. Our analysis focuses on the stochastic fluctuations caused by the finite number of sampling points. Since the target quantity is of order $\epsilon$, these fluctuations must be significantly smaller than $\epsilon$ to ensure that they do not interfere with the estimation of the desired result.

\begin{theorem}\label{thm:var}
Take $\mathcal{X}=\left\{x_i\right\}_{i=1}^n$ and $\mathcal{Z}=\left\{z_k\right\}_{i=1}^m$, where $m=\lceil n^\gamma\rceil$ for some $0<\gamma \leq 1$ and $\lceil x\rceil$ is the smallest integer greater than or equal to $x\in\mathbb{R}$. Let $u_i$ is defined in Section \ref{sec:setup}. Take $X \in C^3\left(\E\right)$ and denote $\boldsymbol{X} \in \mathbb{R}^{nq}$ such that $\boldsymbol{X}$ satisfied (\ref{eq:disc_amb_vf}). Let $\epsilon=\epsilon(n)$ so that $\frac{\sqrt{\log n}}{n^{\gamma / 2} \epsilon^{d/4+1}} \rightarrow 0$ and $\epsilon \rightarrow 0$ when $n \rightarrow \infty$. Then with probability higher than $1-\mathcal{O}\left(1 / n^2\right)$, we have
\begin{equation*}
\frac{1}{\epsilon}\left[\left(I_{nq}-\left(\mathbf{D}_{\beta,\alpha}\right)^{-1} \mathbf{S}_{\beta,\alpha}\right) \boldsymbol{X}\right][i]=u^{-1}_i\frac{X\left(x_i\right)-T_{\epsilon,\beta,\alpha}X(x_i)}{\epsilon} +\mathcal{O}\left(\frac{\sqrt{\log n}}{n^{\gamma / 2} \epsilon^{d/4+1}}\right)\,,
\end{equation*}
for all $i=1,2, \ldots, n$.
\end{theorem}

The proof is deferred to Section \ref{sec:proof}. Note that all quantities on the left-hand side can be stored and computed numerically, whereas those on the right-hand side are abstract in nature; the two are linked via $u_i^{-1}$.

\section{Simulation Experiment Results}\label{sec:simu}

In this section, we use simulated data to validate the theoretical behavior of LA-VDM and compare it with the vanilla VDM. All experiments were conducted on a environment equipped with a 36-core CPU, with 60 GB of RAM. The implementation was done using Python 3.10, and the experimental design and analysis were carried out using NumPy 1.26.4. We optimized memory usage by selecting \texttt{np.float64}.

In the following simulation, $\epsilon$ is selected such that each point's $\sqrt{\epsilon}$-neighborhood contains sufficient number of neighbors (typically more than 30) to ensure stable kernel estimates. 

In addition to comparing the computational time, we also evaluate how accurately LA-VDM approximates VDM. We consider two quantities. Denote $(\lambda_l, w_l)$ as the $l$-th eigenpair of LA-VDM, and $(\mu_l, v_l)$ as the $l$-th eigenpair of VDM, where $w_l$ and $v_l$ are unit vectors. First, the $l$-th eigenvalue difference is quantified by $|\lambda_l-\mu_l|/\mu_l$. Second, how accurate the $l$-th eigenvector is recovered is quantified by the cosine similarity $\frac{w_l}{\|w_l\|} \cdot \frac{v_l}{\|v_l\|}$. 
Third, the difference of the $l$-th vector field obtained by VDM and LA-VDM is quantified by 
$$
\min_{\sigma_l\in\{1,-1\}}\sqrt{\sum_{i=1}^n | w_l[i] - \sigma_l v_l[i] |^2},
$$
where $w_l[i]$ and $v_l[i]$ denote the $i$-th block of the eigenvectors $w_l$ and $v_l$, respectively, and $\sigma_l \in \{1, -1\}$ is used to align the sign of the eigenvectors.

\subsection{Effective landmark kernel}\label{section: effective kernel numerics}

We illustrate the effective landmark kernel and show how parallel transport is well approximated via landmark using the following 2-dimensional manifold.
Consider a 2-dimensional smooth, fully asymmetric surface $\mathbb{S}$ embedded in $\mathbb{R}^3$ that is diffeomorphic to $S^2$. The surface is parametrized by $\mathbb{S}=\varphi(\mathcal{R})$, where $\mathcal{R}=[0,\pi)\times [0,2\pi)$,  
\begin{equation*}
\mathbb{S}=\left\{(x,y,z)\left\vert
\begin{aligned}
& x(u, v)=a\, \rho(u, v) \sin u \cos v\\
& y(u, v)=b\, \rho(u, v) \sin u \sin v\\
& z(u, v)=c\, \rho(u, v) \cos u \\
\end{aligned},
\right.\,  \text{ with  $u \in [0, \pi)$ and $v \in [0, 2\pi)$} \right\},
\end{equation*}
$a=1.1$, $b=1.0$, $c=0.9$, 
\[
\rho(u, v) = 1 + w(u) \left( 
    0.2 \sin(2u + v) + 
    0.15 \cos(4v + u) + 
    0.1 \sin(4u) \cos(2v) 
\right)\,,
\]
and
\[
w(u)=\exp \left(-\frac{1}{u^{2.6}(\pi-u)^{2.2}}\right)\left(\frac{1}{2}+\frac{1}{2} \cos \left(2u-\pi\right)\right)^{2.2} \,.
\]
Here, $w(u)$ acts as a cutoff function, effectively regularizing the non-smooth region around the pole, while the term inside the latter parentheses controls the rate of change so that it does not vary too rapidly.

Now, choose two points,
$s_1 = \varphi(\tfrac{\pi}{2},\, \pi+0.15)$ and
$s_2 = \varphi(\tfrac{\pi}{2},\, \pi-0.15)$ from $\mathbb{S}$,
and sample $m$ landmark points from an $\epsilon=0.3$ neighborhood around each point. The goal is to examine whether the approximation quality of the effective transport operator improves as $m$ increases. At $s_1$, we select the tangent vector $u_{s_1} = (-0.06,\; 0.59,\; -0.80)$, normalized to unit length. Using the tangent space construction and parallel transport approximation described in the last part of Section~\ref{sec:setup} (see also~\cite{singer2011}), we obtain its parallel transport to $s_2$, as $u_{s_2} = (0.04,\;-0.69,\;-0.72)$. We then evaluate how well the approximated effective transport operator, computed using different numbers of landmarks $m=20,\;40,\;60,\;80$. For each value of $m$, the experiment is repeated 30 times, and we report the median and MAD of the resulting $L^2$-norm errors. The results are summarized as follows: for increasing numbers of landmarks $m=20,\;40,\;60$, the median $L^2$ error (median $\pm $MAD) decreases from $0.145 \pm 0.0718$, $0.061 \pm 0.032$ and $0.036 \pm 0.024$, respectively. We also visualize the transported vectors for $m=20$ and $m=40$ in Figure~\ref{fig:transp_vis}. These results clearly indicate that the approximation improves as the number of landmarks increases.

\begin{figure}
    \centering
\includegraphics[trim=0 100 0 150, clip, width=0.6\linewidth]{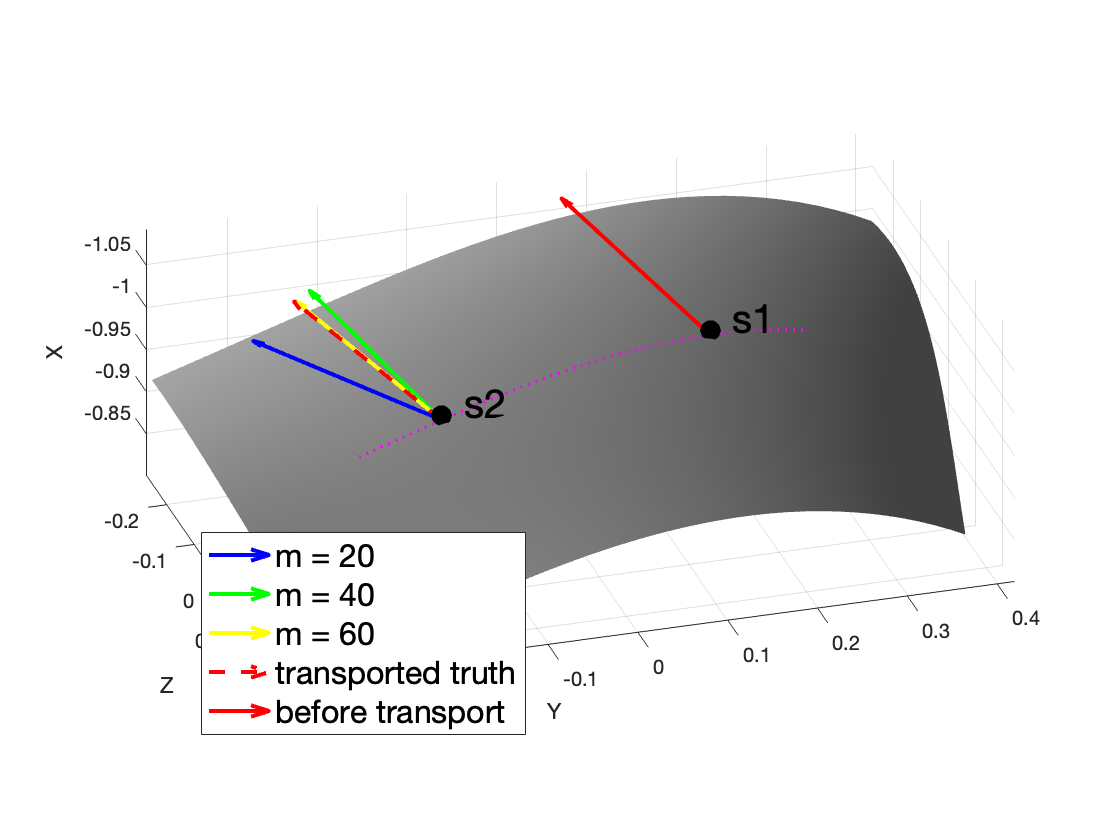}
    \caption{Illustration of parallel transport from $s_1$ to $s_2$ on the distorted sphere. The vector attached at $s_1$ is transported along the surface to $s_2$ using different size of landmark.}
    \label{fig:transp_vis}
\end{figure}

\subsection{The impact of landmark size on LA-VDM}\label{sec:klein}
We study the computational time and how well LA-VDM approximates VDM with various landmark sizes. Consider a Klein bottle, defined as 
\begin{equation*}
\mathbb{K}=\phi(\mathcal{P}),
\end{equation*}
where $\mathcal{P}:=[0,2\pi)\times [0,2\pi)$ is the parameter space of the Klein bottle,
\begin{equation*}
\phi:(u,v)\mapsto\left(
\begin{aligned}
& x(u, v)=R \left( \cos(\frac{u}{2})\cdot \cos(v) - \sin(\frac{u}{2})\cdot \sin(2 v)\right)\\
& y(u, v)=R \left( \sin(\frac{u}{2})\cdot \cos(v) - \cos(\frac{u}{2})\cdot \sin(2 v)\right)\\
& z(u, v)=P\cos(u)(1+\gamma \sin(v)) \\
& w(u, v)=P\sin(u)(1+\gamma \sin(v))
\end{aligned}
\right)\in \mathbb{R}^4\,,
\end{equation*}
and we choose 
$R=2$, $P=1$ and $\gamma=0.1$. We treat $\mathbb{K}$ as $\iota(\M)$. We uniformly sample 3500 points on $\mathbb{K}$, and collect these points as $\{s_i\}_{i=1}^{n}$. We consider different landmark sizes as $m=\lfloor 2^{i/2}\rfloor$ for $i=11,12,\cdots,20$, where $\lfloor\cdot \rfloor$ is the floor function. The above process, which includes generating sample points and landmark points, is repeated independently for 30 times. We focus on $\alpha=0$, $\beta=1/2$ and set $\epsilon=0.2$ since $(1/2, 0)$-LA-VDM gives us the best approximation of VDM according to Corollary \ref{cor:b12}. 
The results are shown in Figure \ref{fig1}. Clearly, when the landmark size increases, the computational time increases, but the top eigenvalues and eigenvectors of $(1/2, 0)$-LA-VDM better approximate those of VDM. 

To further understand how LA-VDM approximates VDM as the landmark size increases, we visualize the vector fields on the parameter space $\mathcal{P}$ of the Klein bottle; that is, for a recovered vector field $X$ on $\mathbb{K}$, we represent it on $\mathcal{P}$ via pulling it back by $\phi^{-1}_*X$ for a visual comparison. 
The pulled back first, third, and fifth eigenvector fields of LA-VDM and VDM to $\mathcal{P}$ are shown in Figure~\ref{fig:exp1_vf}. 

To more closely evaluate the differences between the outputs of VDM and LA-VDM, we consider three pointwise metrics, including the relative $L^2$ error, the cosine similarity, and the relative magnitude difference. For the $l$-th eigenvector from LA-VDM and VDM, denoted by $w_l$ and $v_l$ respectively, consider
\begin{equation*}
I_{l,2}(i):=\frac{\| w_l[i] - v_l[i] \|_2}{\| v_l[i] \|_2}, \ \
I_{l,a}(i):=\frac{ w_l[i] \cdot v_l[i]}{\| v_l[i] \|_2 \| w_l[i] \|_2}, \ \ \mbox{and}\ \ 
I_{l,m}(i):=\frac{\left| \| v_l[i] \|_2 - \| w_l[i] \|_2 \right|}{\| v_l[i] \|_2}.
\end{equation*}
To avoid numerical instability, we omit values of $I_{l,2}(i)$ and $I_{l,m}(i)$ from visualization when the denominator $\| v_l[i] \|_2$ falls below $10^{-7}$.
These indices quantify how the discrepancy between LA-VDM and VDM changes as the landmark set size increases. We plot all metrics on the $(u,v)$-plane. See Figures~\ref{fig:exp1_rel_err}, \ref{fig:exp1_mag_err}, and \ref{fig:exp1_angle_err}. Table~\ref{tab:eigen_landmark_metrics} summarizes the median and median absolute deviation (MAD) of the three metrics $I_{l,2}$, $I_{l,a}$, $I_{l,m}$ computed over all data points and grouped by eigenvector. 
To assess statistical significance, we applied the Wilcoxon signed-rank test to each pair of consecutive landmark sizes, with Bonferroni correction. Since three evaluation metrics were tested under three different landmark sizes $m$, this resulted in 18 pairwise comparisons; thus the significance level was adjusted to $\alpha/18\approx 0.0028$, where $\alpha$ is set to 0.05. Results marked with an asterisk (*) in Table~\ref{tab:eigen_landmark_metrics} indicate $p<\alpha/18$.
Across all eigenvectors and all metrics, the discrepancy between LA-VDM and VDM consistently decreases as the number of landmarks increases, fully in line with our theoretical guarantees.

\begin{figure}[h]
\begin{center}
 \includegraphics[width=1.\textwidth]{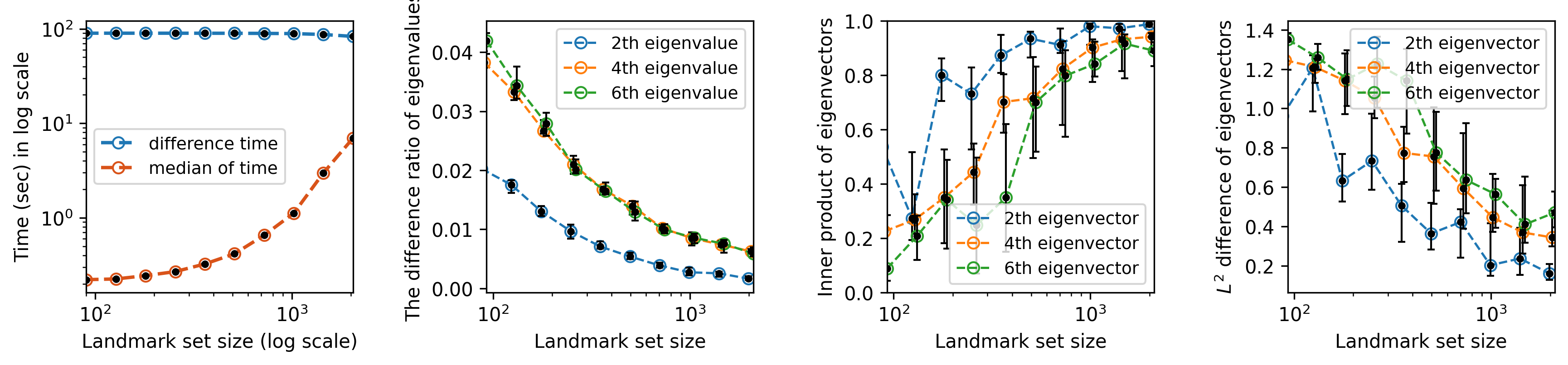} 
\end{center}
\caption{From left to right: a comparison of computational time between VDM and $(1/2,0)$-LA-VDM with different landmark sizes, he difference ratio of the 1st, 3rd, 5th eigenvalues between VDM and $(1/2,0)$-LA-VDM, the cosine similarity of the 1st, 3rd, 5th eigenvectors between VDM and $(1/2,0)$-LA-VDM.}
\label{fig1}
\end{figure}

\begin{figure}[h]
\begin{center}
 \includegraphics[width=1.\textwidth]{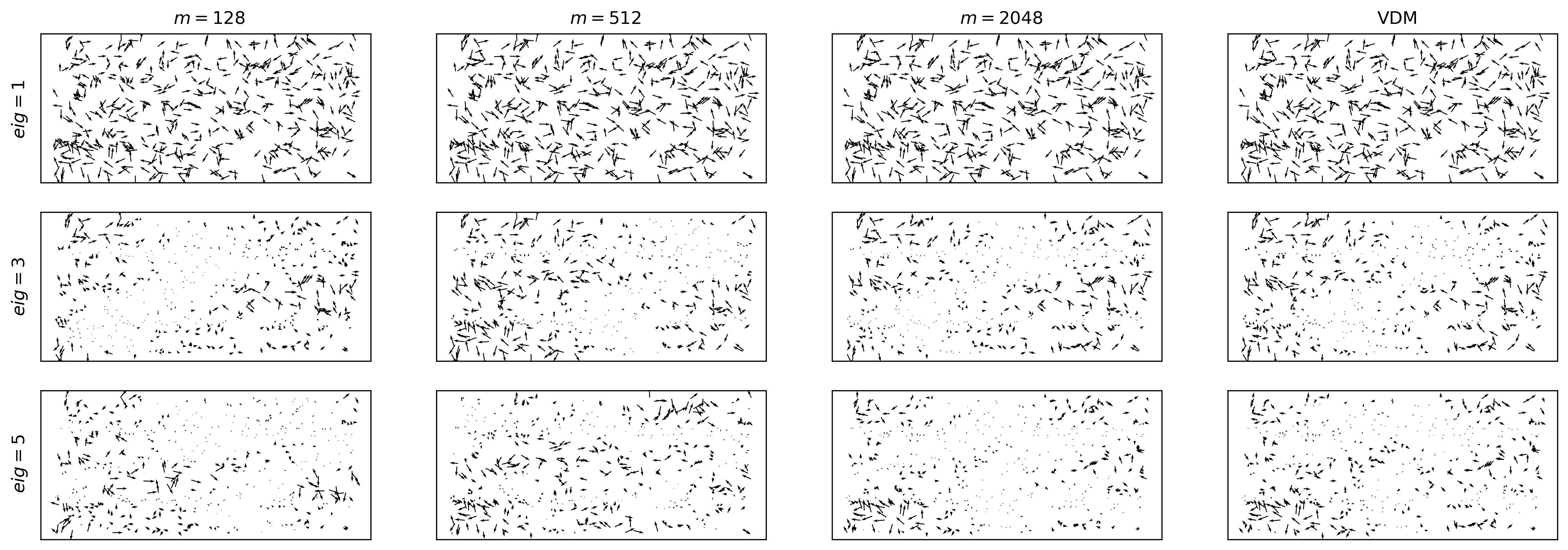} 
\end{center}
\caption{Plot of the first, third, and fifth eigenvector fields eigenvectors, for the Klein bottle on the parameter space $\mathcal{P}$. From top to bottom: first, third, and fifth eigen vector fields, and from left to right column: landmark set size 256, 1024, and 4096.}
\label{fig:exp1_vf}
\end{figure}

\begin{figure}[h]
\begin{center}
   \includegraphics[trim=5 0 735 0, clip, width=0.325\textwidth]{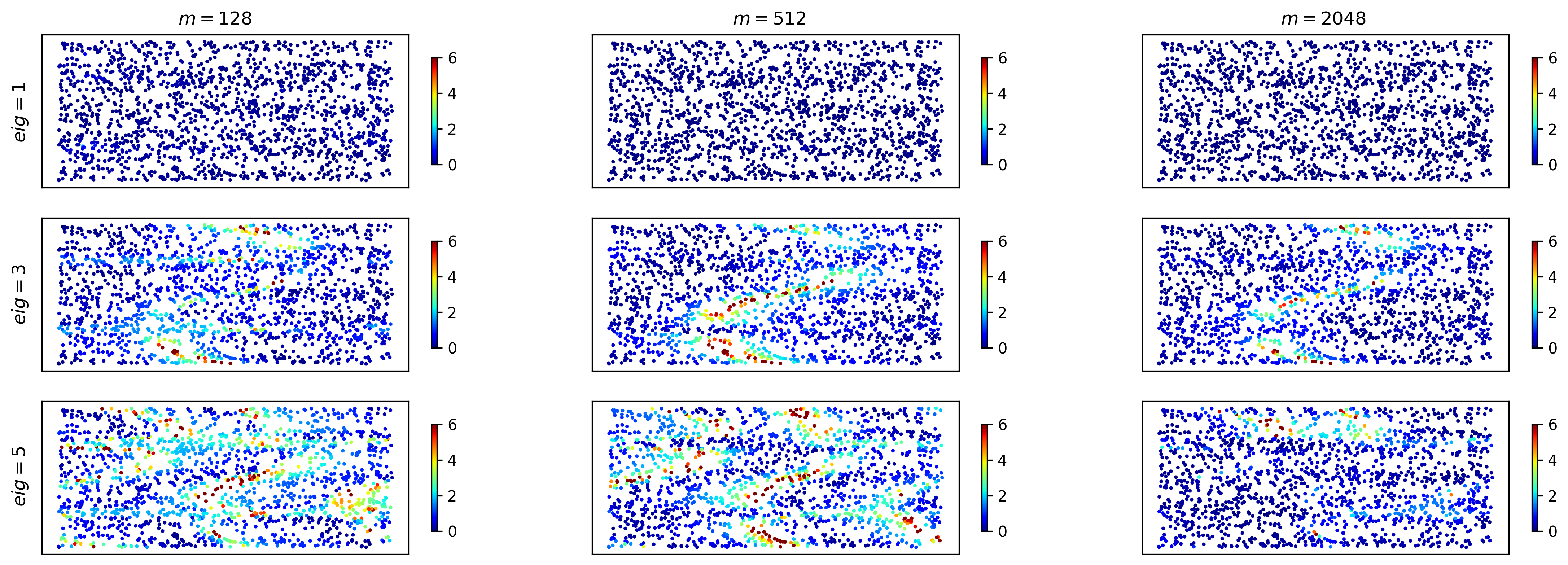} 
 \includegraphics[trim=370 0 370 0, clip, width=0.325\textwidth]{Figures/exp1_rel_error_3.png} 
 \includegraphics[trim=735 0 5 0, clip, width=0.325\textwidth]{Figures/exp1_rel_error_3.png} 
\end{center}
\caption{Plot of $I_{l,2}$, where $l=1,3,5$, for the Klein bottle on the parameter space $\mathcal{P}$; that is, plot $(\phi^{-1}(s_i), I_{l,2}(i))$. 
From top to bottom: first, third, and fifth eigen vector fields, and from left to right column: landmark set size 256, 1024, and 4096.}
\label{fig:exp1_rel_err}
\end{figure}

\begin{figure}[h]
\begin{center}
  \includegraphics[trim=5 0 735 0, clip, width=0.325\textwidth]{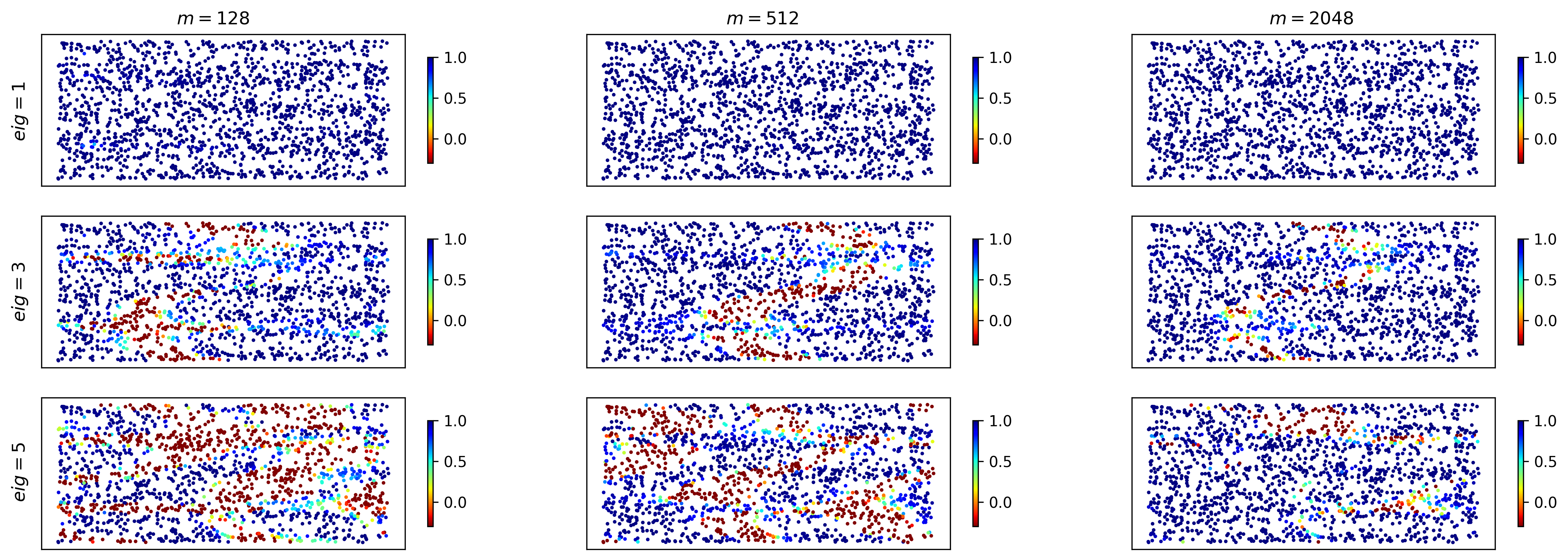} 
 \includegraphics[trim=370 0 370 0, clip, width=0.325\textwidth]{Figures/exp1_inner_prod_3.png} 
 \includegraphics[trim=735 0 5 0, clip, width=0.325\textwidth]{Figures/exp1_inner_prod_3.png} 
\end{center}
\caption{Plot of $I_{l,a}$, where $l=1,3,5$, for the Klein bottle on the parameter space $\mathcal{P}$; that is, plot $(\phi^{-1}(s_i), I_{l,a}(i))$. From top to bottom: first, third, and fifth eigen vector fields, and from left to right column: landmark set size 256, 1024, and 4096.}
\label{fig:exp1_mag_err}
\end{figure}

\begin{figure}[h]
\begin{center}
 \includegraphics[trim=5 0 735 0, clip, width=0.325\textwidth]{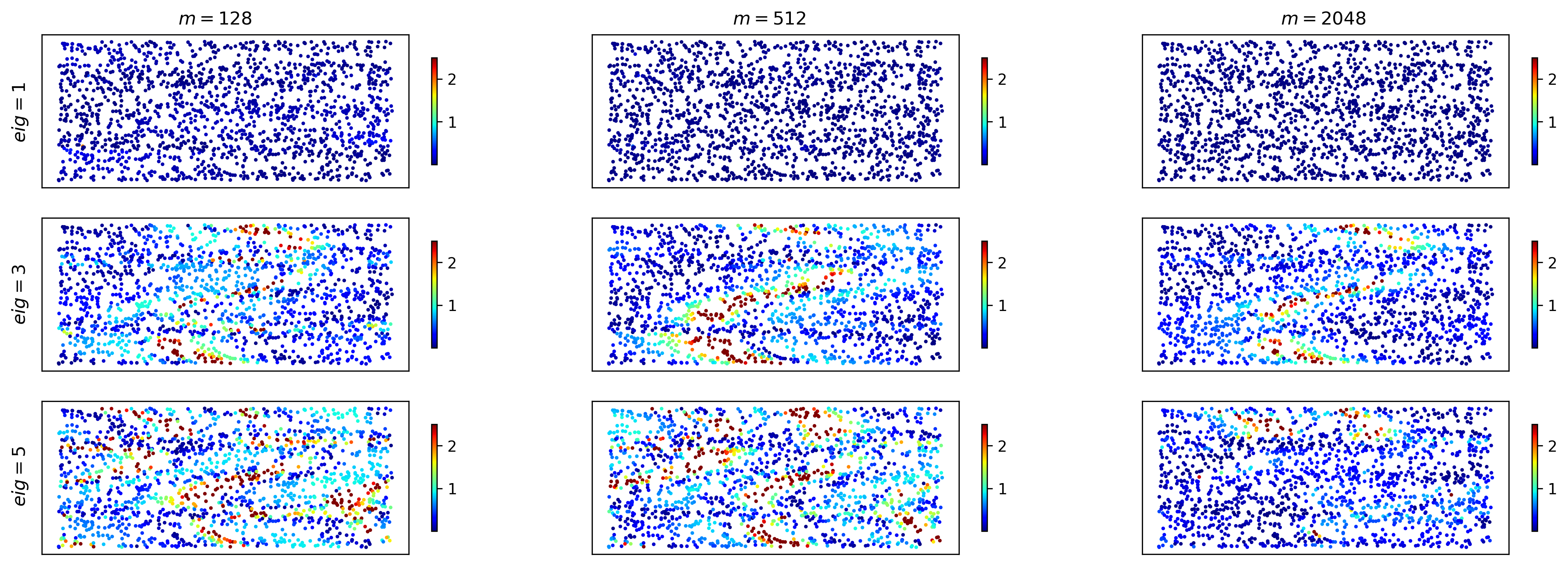}
 \includegraphics[trim=370 0 370 0, clip, width=0.325\textwidth]{Figures/exp1_mag_error_3.png}
 \includegraphics[trim=735 0 5 0, clip, width=0.325\textwidth]{Figures/exp1_mag_error_3.png}
\end{center}
\caption{Plot of $I_{l,m}$, where $l=1,3,5$, for the Klein bottle on the parameter space $\mathcal{P}$; that is, plot $(\phi^{-1}(s_i), I_{l,m}(i))$. From top to bottom: first, third, and fifth eigen vector fields, and from left to right column: landmark set size 256, 1024, and 4096.}
\label{fig:exp1_angle_err}
\end{figure}

\begin{table}[h]
\small
\centering
\begin{tabular}{
>{\centering\arraybackslash}m{1.cm} 
>{\centering\arraybackslash}m{1.cm} 
*3{>{\centering\arraybackslash}m{2.2cm} }
}
\noalign{\hrule height 1.2pt}
\multirow{2}{*}{Eigen} & \multirow{2}{*}{Metric}  & \multicolumn{3}{c}{Landmark size} \\
 && 128 & 512 & 2048 \\
\midrule
\multirow{3}{*}{1st} 
 & $I_{1,2}$ &0.108$\pm$0.045 & 0.036$\pm$0.014* & 0.026$\pm$0.010* \\
 & $I_{1,a}$  &0.639$\pm$0.378 & 0.549$\pm$0.339* & 0.342$\pm$0.204* \\
 & $I_{1,m}$  &1.262$\pm$0.686 & 1.029$\pm$0.676* & 0.303$\pm$0.179* \\
\hline
\multirow{3}{*}{3rd} 
 & $I_{3,2}$ &0.999$\pm$0.001 & 0.999$\pm$0.001* & 0.999$\pm$0.001* \\
 & $I_{3,a}$  &0.980$\pm$0.021 & 0.984$\pm$0.016 & 0.997$\pm$0.003 \\
 & $I_{3,m}$  &0.524$\pm$0.475 & 0.892$\pm$0.108* & 0.997$\pm$0.003* \\
\hline
\multirow{3}{*}{5th} 
 & $I_{5,2}$ &0.067$\pm$0.039 & 0.019$\pm$0.012* & 0.010$\pm$0.006* \\
 & $I_{5,a}$  &0.374$\pm$0.244 & 0.369$\pm$0.249* & 0.288$\pm$0.178* \\
 & $I_{5,m}$  &0.649$\pm$0.325 & 0.592$\pm$0.342* & 0.242$\pm$0.143* \\
\noalign{\hrule height 1.2pt}
\end{tabular}
\caption{Performance comparison across different eigenvector with different landmark sizes.}
\label{tab:eigen_landmark_metrics}
\end{table}

\subsection{The impact of landmark distribution and $\beta$-normalization}\label{sec:asymsphere}

To study the impact of landmark distribution, consider the distorted sphere, $\mathbb{S}$, defined in Section \ref{section: effective kernel numerics} and take $\epsilon=0.17$. We sample 3,500 points on $\mathbb{S}^2$ from an angular central Gaussian distribution with density $p(x) \propto(x^\top \Sigma^{-1} x)^{-3/2}$ where $x\in\mathbb{S}^2$ and $\Sigma=\mathtt{diag}(1,1,0.8)$, and map them to a distorted sphere $\mathbb{S}$, resulting in a non-uniform
point set $\tilde{\mathcal{X}}$.
We then uniformly select 2,200 landmark points from $\tilde{\mathcal{X}}$, denoted by $\tilde{\mathcal{Z}}$. To focus on the effect of $\beta$, we fix $\alpha = 0$ for both VDM and LA-VDM, which means no correction for sampling density is applied. However, since both methods are applied to the same dataset, the influence of sampling density remains consistent across the two embeddings. We investigate whether an appropriate choice of $\beta$ allows the $(\beta, 0)$-LA-VDM to recover the vanilla VDM, as suggested by Corollary~\ref{cor:b12}. The quantification results with various $\beta$ are illustrated in Figure \ref{fig:b12}. We observe that $(1/2,0)$-LA-VDM best approximates the embedding by VDM, which is consistent with our theoretical results.

\begin{figure}[h]
    \centering
    \includegraphics[width=0.8\linewidth]{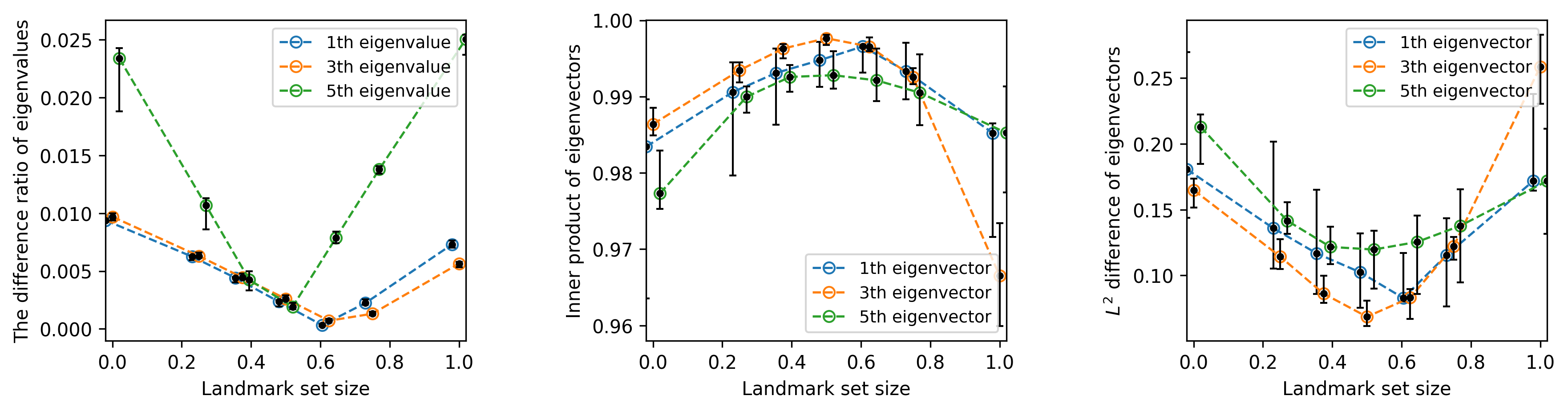}
    \caption{From left to right: the difference ratio of the 2nd, 4th, 6th non-trivial eigenvalues between VDM and $(\beta, 0)$-LA-VDM,  the inner product between the 2nd, 4th, 6th non-trivial eigenvectors of  VDM and $(\beta, 0)$-LA-VDM, the difference of the 2nd, 4th, 6th non-trivial eigenvectors of VDM and $(\beta, 0)$-LA-VDM.}
    \label{fig:b12}
\end{figure}

\subsection{The impact of $\alpha$-normalization}
To study the impact of $\alpha$-normalization, we use the same smooth fully asymmetric surface $ \mathbb{S}$. We generate a non-uniform sample of 3,500 points on $\mathbb{S}$, from which 2,500 landmark points are uniformly selected.
In this experiment, we aim to validate Corollary~\ref{cor:a1}, which states that setting $\alpha =1$ eliminates the influence of sampling density on the embedding. We fix $\alpha = 1$ in the vanilla VDM, for which the resulting embedding is known to be independent of the sampling density \cite{singer2011}. According to Corollary~\ref{cor:a1}, we set $\beta = 1/2$ in LA-VDM and vary $\alpha$ from 0 to 1. We then compare the embedding vector field produced by LA-VDM and the embedding of VDM with $\alpha = 1$. As shown in Figure~\ref{fig:a1}, the difference between the two embeddings becomes smaller as $\alpha$ approaches 1. This confirms that when $\alpha = 1$ and $\beta = 1/2$, the LA-VDM embedding becomes robust to the sampling density. Consequently, by incorporating $\alpha$-normalization, our LA-VDM successfully addresses the issue of sampling density that Roseland fails to eliminate.

\begin{figure}[h]
    \centering
    \includegraphics[width=0.8\linewidth]{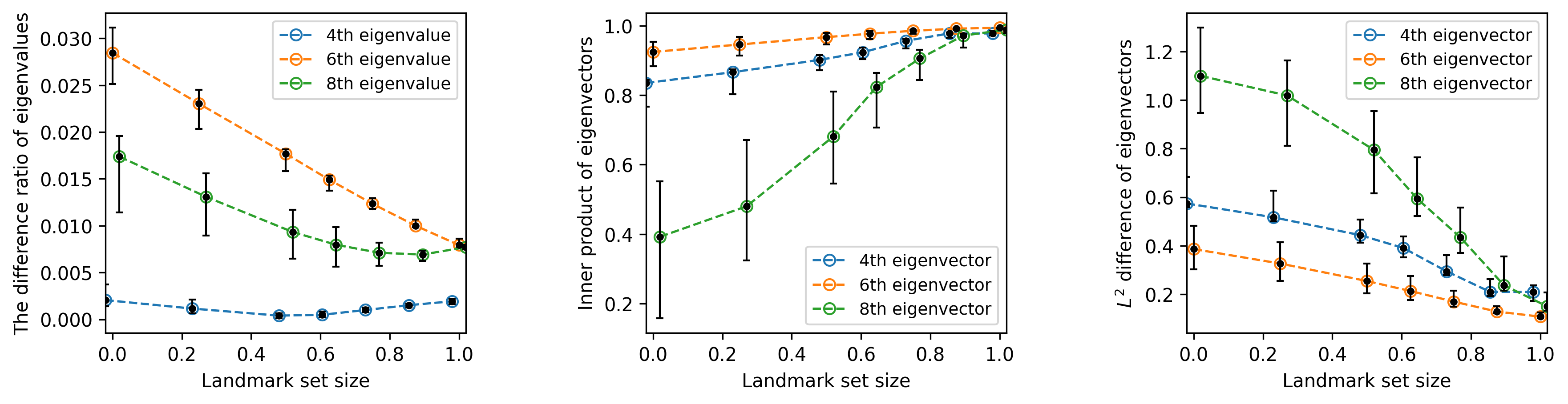}
    \caption{From left to right: the difference ratio of the 2nd, 4th, 6th non-trivial eigenvalues between VDM with $\alpha=1$ and $(1/2, \alpha)$-LA-VDM,  the inner product between the 2nd, 4th, 6th non-trivial eigenvectors of VDM with $\alpha=1$ and $(1/2, \alpha)$-LA-VDM, the difference of the 2nd, 4th, 6th non-trivial eigenvectors of VDM with $\alpha=1$ and $(1/2, \alpha)$-LA-VDM.}
    \label{fig:a1}
\end{figure}

\subsection{Eigenpair Recovery}
To further assess the recovery capability of $(\beta,\alpha)$-LA-VDM, we conducted an experiment on the distorted sphere $\mathbb{S}$ introduced in Section~\ref{sec:asymsphere}. We uniformly sampled $n=5,000$ points from $\mathbb{S}$ and randomly selected $500\approx 7\sqrt{n}$ landmark points. Set $\alpha=0$ and $\beta=1/2$. Our goal was to evaluate how well LA-VDM recovers the top 20 eigenvalues and eigenvectors, using the VDM results on the full dataset as the ground truth. 

The accuracy of recovered eigenvalues and eigenvectors is shown in Figure~\ref{fig:fig_4_recover}, where we repeated the experiment on 30 random datasets to compute the error bars. Visually, $(\beta,\alpha)$-LA-VDM successfully recover the top six eigenvectors of VDM, consistent with the theoretical guarantees. 

\begin{figure}
    \centering
    \includegraphics[width=0.8\linewidth]{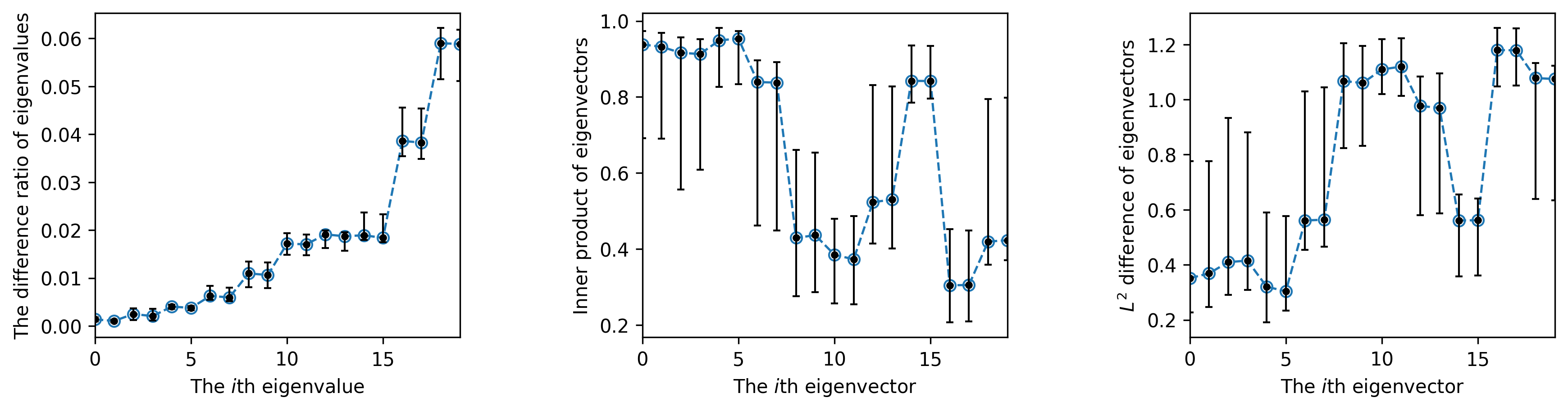}
    \caption{From left to right: the difference ratio of eigenvalues between VDM and LA-VDM, the inner product between eigenvectors of VDM  and LA-VDM, the $L^2$ difference of eigenvectors of VDM and LA-VDM.}
    \label{fig:fig_4_recover}
\end{figure}

\subsection{Results with large-scale data}

To assess the scalability of the algorithm, we conducted experiments on large datasets using the same hardware setup, Python environment, and NumPy version as described earlier. To handle the increased data size, we optimized memory usage by adopting the \texttt{np.float32} data type and employing sparse matrix representations.

We evaluated the proposed LA-VDM method on two large-scale datasets: the Klein bottle $\mathbb{K}$ and the distorted sphere $\mathbb{S}$ defined in Section~\ref{sec:klein} and \ref{sec:asymsphere}, respectively. For both datasets, we used sample sizes of 500,000 and 1,000,000. 
For the Klein bottle $\mathbb{K}$ (distorted sphere $\mathbb{S}$ respectively), we used bandwidths $\epsilon = 0.03$ and $0.025$ ($\epsilon = 0.014$ and $0.012$ respectively), the landmark sets contained 706 and 1000 points (approximately $\sqrt{n}$),
landmark affinity matrices were constructed by truncating entries beyond $5.0\epsilon$, and the computation times of LA-VDM were 330.7 s and 782.9 s (281.8 s and 710.6 s respectively).

The original VDM could not be fully tested because of prohibitive memory requirements. Even after sparsifying the affinity matrix by truncating distances above $\epsilon$, resulting in only about 0.17\% nonzeros entries, the case $n = 1,000,000$ still exceeded memory capacity. For $500{,}000$ points, the VDM eigendecomposition requires roughly 50 minutes. In contrast, LA-VDM is dramatically faster and far more memory-efficient, enabling computations that are infeasible for standard VDM.


\section*{Acknowledgments}
The successful completion of this research was made possible by the academic resources and advanced research infrastructure provided by the National Center for High-Performance Computing, National Institutes of Applied Research (NIAR), Taiwan. We gratefully acknowledge their invaluable support. This research was funded in part by National Science and Technology Council grant 112-2115-M-002-015-MY3. The authors declare no competing interests.

\bibliographystyle{siamplain}
\bibliography{main}

\appendix

\clearpage

\setcounter{page}{1}
	\setcounter{equation}{0}
	\setcounter{figure}{0}
	\renewcommand{\thepage}{SI.\arabic{page}}
	\renewcommand{\thesection}{SI.\arabic{section}}
	\renewcommand{\theequation}{SI.\arabic{equation}}
	\renewcommand{\thelemma}{SI.\arabic{lemma}}
	\renewcommand{\thetable}{SI.\arabic{table}}
	\renewcommand{\thefigure}{SI.\arabic{figure}}

\section{Necessary differential geometry background}\label{section: necessary diff geo background}
%

We mainly follow the notation in \cite{berline2003heat} and focus on connection-related topics in the principal and vector bundle setting. We assume basic knowledge of Riemannian geometry.

Let $\mathcal{F}$ and $\M$ be smooth manifolds. From now on we assume that $\M$ is compact without boundary, although the discussion extends to more general situations (e.g.\ paracompact $\M$). A fiber bundle over $\M$ with typical fiber $F$ is a pair $(\mathcal{F},\pi)$, where the \emph{canonical projection} $\pi:\mathcal{F}\to \M$ is a smooth map and $\mathcal{F}$ is locally trivial; that is, there exists an open covering $\{U_i\}$ of $\M$ and diffeomorphisms $\phi_i:\pi^{-1}(U_i)\to U_i\times F$ such that $\pi:\pi^{-1}(U_i)\to U_i$ is the composition of $\phi_i$ with the projection onto the first factor $U_i$. We call $\mathcal{F}$ the \emph{total space} and $\M$ the \emph{base}. For each $x\in \M$, the fiber $\mathcal{F}_x:=\pi^{-1}(x)$ is a closed submanifold of $\mathcal{F}$ diffeomorphic to $F$, and we call $F$ the fiber of $\mathcal{F}$.

A principal bundle with structure group $G$ (a Lie group) is a fiber bundle $(P,\pi)$ equipped with a right action of $G$ on the fibers; that is, $(p\circ g)\circ h=p\circ(gh)$ for $p\in P$ and $g,h\in G$, and $\pi(p\circ g)=\pi(p)$ for all $p\in P$ and $g\in G$. The right action is required to be free and transitive on each fiber. Equivalently, for each local trivialization $\phi_i$ we have $\phi_i(p\circ g)=\phi_i(p)\circ g$ for all $p\in \pi^{-1}(U_i)$ and $g\in G$. In particular, each fiber $\pi^{-1}(x)$ is diffeomorphic to $G$, and $\M$ can be identified with the quotient manifold $P/G$. We denote such a bundle by $P:=P(\M,G)$, and sometimes call it a $G$-bundle when no confusion is possible. If $P=\M\times G$ with $(x,g)\circ h=(x,gh)$ for $x\in \M$ and $g,h\in G$, then $P$ is called \emph{trivial}. A typical nontrivial example is $SO(3)$, which is a principal $SO(2)$-bundle over $S^2$.

A vector bundle is a fiber bundle whose typical fiber is a vector space $E$, and such that the diffeomorphisms $\phi_j\circ \phi_i^{-1}:\{x\}\times E\to \{x\}\times E$ are invertible linear maps for all $x\in U_i\cap U_j$. In this paper we consider only real vector bundles. A typical example is the tangent bundle of $S^2$, denoted by $TS^2$, whose typical fiber is $\mathbb{R}^2$.

In our work, we model complex data structure by taking into account the relationship between principal bundles and vector bundles, which encodes the familiar correspondence between bases and coordinates from linear algebra. Suppose $P$ is a principal bundle with structure group $G$ and $E$ is an $N$-dimensional vector space carrying a linear representation of $G$. Then one can form a vector bundle over $\M$ with typical fiber $E$ via the associated bundle $P\times_G E$, defined by the equivalence relation $(p\circ g,v)\sim (p,g^{-1}v)$ for all $p\in P$, $g\in G$, and $v\in E$. Geometrically, $(p,v)$ records a base point, a choice of frame, and the corresponding coordinate vector, while the equivalence relation encodes the usual change-of-coordinates rule.

Conversely, every rank-$N$ vector bundle $\mathcal{E}$ over $\M$ can be constructed from a principal bundle. Specifically, let $GL(\mathcal{E})$ be the bundle whose fiber over $x\in \M$ consists of all invertible linear maps from $\mathbb{R}^N$ to $\mathcal{E}_x$. Then $GL(\mathcal{E})$ is a principal $GL(N)$-bundle under the action $(p\circ g)(v)=p(gv)$ for $p\in GL(\mathcal{E})$, $g\in GL(N)$, and $v\in \mathbb{R}^N$, and one has an isomorphism $\mathcal{E}\cong GL(\mathcal{E})\times_{GL(N)}\mathbb{R}^N$. We call $GL(\mathcal{E})$ the \emph{frame bundle} of $\mathcal{E}$. The same geometric interpretation applies: $(p,v)\in GL(\mathcal{E})\times \mathbb{R}^N$ records a frame $p$, the base point $x=\pi(p)$, and the coordinate vector $v$, and the relation $(p\circ g)(v)=p(gv)$ encodes the change-of-coordinates rule.

We now consider a typical example used in this paper. Let $\mathcal{E}=T\M$ be the tangent bundle of a $d$-dimensional manifold $\M$. A linear frame at $x\in \M$ is an ordered basis of $T_x\M$. Let $L(\M)$ denote the set of all linear frames at all points of $\M$. Then $L(\M)$ is a principal $GL(d)$-bundle with canonical projection $\pi$ sending a frame at $x$ to $x$, and with the right action of $GL(d)$ given by $u\circ g=\sum_j g_i^j u_j$, where $g=(g_i^j)\in GL(d)$ and $u=(u_1,\ldots,u_d)$ is a basis of $T_x\M$. Thus $L(\M)=GL(T\M)$. Moreover, if $u\in L(\M)$ with $\pi(u)=x$, then $u$ defines an invertible linear map $\mathbb{R}^d\to T_x\M$ via
$$
u(v)=\sum_{j=1}^d v_j u_j,
\qquad v=(v_1,\ldots,v_d)\in \mathbb{R}^d.
$$
Hence $T\M$ is the vector bundle associated to $GL(T\M)$ with typical fiber $\mathbb{R}^d$. More general tensor bundles can be constructed similarly; for example, taking $E=\Lambda(\mathbb{R}^d)^*$ yields the bundle of exterior forms $\Lambda T^*\M$.

For a vector bundle $\mathcal{E}$ over $\M$, a map $s:\M\to \mathcal{E}$ is called a \emph{section} if $\pi(s(x))=x$ for all $x\in \M$. We denote by $\Gamma(\mathcal{E})$ the space of smooth sections, and by $C^k(\mathcal{E})$ the space of $k$-times differentiable sections ($k\ge 0$). Let $C(\mathcal{E}):=C^0(\mathcal{E})$ be the space of continuous sections. Note that one can define sections of a principal bundle $P(\M,G)$ in the same way; a smooth global section $\sigma:\M\to P$ exists if and only if $P$ is trivial (i.e.\ $P\cong \M\times G$).

A metric on a vector bundle $\mathcal{E}$, denoted by $g^{\mathcal{E}}$, is a smooth family of non-degenerate inner products on the fibers of $\mathcal{E}$. By smooth we mean that for smooth sections $s_1,s_2$, the function $g^{\mathcal{E}}(s_1,s_2)$ is smooth on $\M$. In general such a metric exists when $\M$ is paracompact, which is satisfied here since $\M$ is compact. Equivalently, $\mathcal{E}$ is associated to an $O(N)$-principal bundle $O(\mathcal{E})\subset GL(\mathcal{E})$ encoding the inner product. With a fiber metric, one can define $L^p(\mathcal{E})$ ($1\le p<\infty$), the space of $L^p$-integrable sections, and $L^\infty(\mathcal{E})$ similarly.

If $\mathcal{E}$ is a vector bundle over $\M$, a \emph{covariant derivative} on $\mathcal{E}$ is a differential operator
$$
\nabla^{\mathcal{E}}: \Gamma(\mathcal{E}) \to \Gamma(T^*\M\otimes \mathcal{E})
$$
satisfying the Leibniz rule: for $s\in \Gamma(\mathcal{E})$ and $f\in C^\infty(\M)$,
$$
\nabla^{\mathcal{E}}(fs)=df\otimes s+f\nabla^{\mathcal{E}}s.
$$
For a vector field $X\in \Gamma(T\M)$, the covariant derivative along $X$ is denoted $\nabla_X^{\mathcal{E}}:=\iota(X)\nabla^{\mathcal{E}}$, where $\iota(X)$ is contraction. Clearly $\nabla^{\mathcal{E}}_{fX}=f\nabla^{\mathcal{E}}_X$ for $f\in C^\infty(\M)$. We mention that $\nabla^{\mathcal{E}}$ corresponds to a connection on the frame bundle $GL(\mathcal{E})$, which provides a smooth splitting of $T(GL(\mathcal{E}))$ into horizontal and vertical subspaces; see \cite[p.~19 and Proposition~1.16]{berline2003heat}.

Let $\nabla^{\mathcal{E}}$ be a covariant derivative on $\mathcal{E}$. Given a $C^1$ curve $\gamma:[0,1]\to \M$, the \emph{parallel transport} along $\gamma$ is the family of linear maps $\parallelslant_{\gamma(0)}^{\gamma(t)}\in \mathrm{Hom}(\mathcal{E}_{\gamma(0)},\mathcal{E}_{\gamma(t)})$ obtained by solving the ODE
$$
\nabla^{\mathcal{E}}_{\dot\gamma(t)} U(t)=0,\qquad U(0)=u\in \mathcal{E}_{\gamma(0)}.
$$
(Here $\nabla^{\mathcal{E}}_{\dot\gamma(t)}$ is defined using the pull-back connection on $\gamma^*\mathcal{E}$.)
For convenience, we set $\parallelslant_y^x u=0$ for $u\in \mathcal{E}_y$ if $y$ lies in the cut locus of $x$.
Moreover, for $s\in C^1(\mathcal{E})$, the covariant derivative at $x=\gamma(0)$ in the direction $v=\gamma'(0)$ can be written as
$$
\nabla^{\mathcal{E}}_{v}s(x)
=
\lim_{h\to 0}\frac{1}{h}\Bigl[\parallelslant^{\gamma(0)}_{\gamma(h)}s(\gamma(h))-s(\gamma(0))\Bigr].
$$
This limit is independent of the choice of $\gamma$ with $\gamma'(0)=v$. See Example~\ref{Example: parallel transport S^2} for a concrete computation.

A covariant derivative $\nabla^{\mathcal{E}}$ preserves the metric $g^{\mathcal{E}}$ if for any $s_1,s_2\in \Gamma(\mathcal{E})$,
$$
d\,g^{\mathcal{E}}(s_1,s_2)
=
g^{\mathcal{E}}(\nabla^{\mathcal{E}} s_1,s_2)
+
g^{\mathcal{E}}(s_1,\nabla^{\mathcal{E}} s_2),
$$
where $d$ is the exterior differential. In this paper we consider only metric-compatible connections.
For the tangent bundle $T\M$ of a Riemannian manifold $(\M,g)$, there is a unique metric-compatible torsion-free
connection, the \emph{Levi--Civita connection}, which we denote by $\nabla$.

Given the Levi--Civita connection $\nabla$ on $T\M$ and a connection $\nabla^{\mathcal{E}}$ on $\mathcal{E}$,
the induced connection on $T^*\M\otimes \mathcal{E}$ is defined by
$$
\nabla^{T^*\M\otimes \mathcal{E}}(\phi\otimes s)
=
\nabla^{T^*\M}\phi\otimes s+\phi\otimes \nabla^{\mathcal{E}}s,
$$
for $\phi\in \Gamma(T^*\M)$ and $s\in \Gamma(\mathcal{E})$.
The \emph{connection Laplacian} $\nabla^2:\Gamma(\mathcal{E})\to \Gamma(\mathcal{E})$ is defined by
$$
\nabla^2:=-\operatorname{tr}\bigl(\nabla^{T^*\M\otimes \mathcal{E}}\nabla^{\mathcal{E}}\bigr),
$$
where $\operatorname{tr}$ denotes contraction over cotangent indices using $g$.
This is the main geometric target of VDM, and hence of LA-VDM. To distinguish it from the Laplace--Beltrami operator
$\Delta:C^\infty(\M)\to C^\infty(\M)$, we reserve $\Delta$ for the latter.

For self-containedness, we now present a concrete example on $\mathbb{S}^2$ illustrating why landmark-constrained
(two-step) parallel transport requires additional care.

\begin{example}\label{Example: parallel transport S^2}

Consider the canonical 2-dim sphere $\mathbb{S}^2$ and we specialize to the tangent bundle and use the Levi-Civita connection associated with the standard Riemannian metric, expressed in terms of Christoffel symbols in a coordinate basis. See Figure \ref{fig:parallel estimation LAVA}.

Consider an open chart of $U\subset \mathbb{S}^2$ with the spherical coordinate $\varphi:(\theta, \phi)\in (0, \pi)\times (-\pi, \pi)\to(\sin(\theta)\cos(\phi),\sin(\theta)\sin(\phi),\cos(\theta))$, where $\theta$ denotes the colatitude and $\phi$ denotes the longitude. The Riemannian metric on $\mathbb{S}^2$ induced from the isometric embedding in $\mathbb{R}^3$ is $g = d\theta^2 + \sin^2\theta\, d\phi^2$.
Denote $\partial_\theta:=\varphi_*e_1=(\cos(\theta)\cos(\phi),\cos(\theta)\sin(\phi),-\sin(\theta))$ and $\partial_\phi:=\varphi_*e_2=(-\sin(\theta)\sin(\phi),\sin(\theta)\cos(\phi),0)$ as coordinate vector fields, where $e_i$ form a canonical basis of $\mathbb{R}^2$, which at each point $x\in U$ form a basis of the tangent space $T_x\mathbb{S}^2$. 
Let the initial tangent vector be $u_0 = a\, \partial_\phi|_x$ at $x = \varphi(\pi/2, 0)$, where $a > 0$ is a constant. Consider three geodesics:
\begin{itemize}
\item $\gamma_1(t) = (\cos(t), \sin(t), 0)$ for $t \in [0, \pi/2]$,
\item $\gamma_{21}(t) = (\cos(t), 0, \sin(t))$ for $t \in [0, \pi/4]$,
\item $\gamma_{22}(t) = \left( \frac{\cos(t)}{\sqrt{2}}, \sin(t), \frac{\cos(t)}{\sqrt{2}} \right)$ for $t \in [0, \pi/2]$.
\end{itemize}
Define two paths from $x$ to $y = \varphi(\pi/2, \pi/2)$:
\begin{itemize}
\item Path $\Gamma_1$ directly connects $x$ to $y$ along $\gamma_1$, which is a geodesic.
\item Path $\Gamma_2$ is composed of $\gamma_{21}$ from $x$ to $z = \varphi(\pi/4, 0)$, followed by $\gamma_{22}$ from $z$ to $y$, which is not a geodesic.
\end{itemize}
By a direct calculation, the Christoffel symbol associated with the spherical coordinate is $\Gamma^\theta_{\phi,\phi}(\varphi(\theta,\phi))=-\sin(\theta)\cos(\theta)$, $\Gamma^{\phi}_{\phi,\theta}(\varphi(\theta,\phi))=\Gamma^\phi_{\theta,\phi}(\varphi(\theta,\phi))=\cot(\theta)$ and other terms are $0$. 
Then, recall that the parallel transport of a vector $V_0\in T_x\mathbb{S}^2$ along a curve $\gamma$ such that $\gamma(0)=x$ and $\gamma'(t)=\dot{\gamma}^\theta(t)\partial_\theta|_{\gamma(t)}+\dot{\gamma}^\phi(t)\partial_\phi|_{\gamma(t)}$, denoted as a vector field $V(\gamma(t))=V^\phi(t)\partial_\phi|_{\gamma(t)}+V^\theta(t)\partial_\theta|_{\gamma(t)}$, in general satisfies $\nabla_{\gamma'}V=0$, which can be spelled out to be $\frac{dV^\theta(t)}{dt}+\Gamma^{\theta}_{\phi,\phi}(\gamma(t))\dot{\gamma}^\phi(t)V^\phi(t)=0$ and $\frac{dV^\phi(t)}{dt}+\Gamma^{\phi}_{\theta,\phi}(\gamma(t))\dot{\gamma}^\theta(t)V^\phi(t)+\Gamma^{\phi}_{\phi,\theta}(\gamma(t))\dot{\gamma}^\phi(t)V^\theta(t)=0$.

First, consider the parallel transport of $X_0\in T_x\M$ along $\Gamma_1$. Along $\gamma_1(t)$, $\partial_\theta|_{\gamma_1(t)}=(0,0,1)$ and $\partial_\phi|_{\gamma_1(t)}=(-\sin(t),\cos(t),0)$, and hence $\dot{\gamma}_1^\theta(t)=0$, $\dot{\gamma}_1^\phi(t)=1$. Denote $X(\gamma_1(t)) = X_1^\theta(t) \partial_\theta|_{\gamma_1(t)} + X_1^\phi(t) \partial_\phi|_{\gamma_1(t)}$ to be the parallel transported vector field along $\Gamma_1$, which satisfies $\frac{d X^\theta_1(t)}{d t}= \frac{d X^\phi_1(t)}{d t}=0$
with initial value $X_1^{\theta}(0)=0$ and $X_1^\phi(0)=a$. Therefore, the parallel transported vector at $y$ is $u_1:=a\partial_\phi|_y$. 

Next, we consider parallel transporting $u_0$ along the composite path $\Gamma_2$. Along $\gamma_{21}$, $\partial_\theta|_{\gamma_{21}(t)}=(\sin(t),0,-\cos(t))$ and $\partial_\phi|_{\gamma_{21}(t)}=(0,\cos(t),0)$, and hence $\dot{\gamma}_{21}^\theta(t)=-1$, $\dot{\gamma}_{21}^\phi(t)=0$. Note that $\theta=t+\pi/2$ in the parametrization. Denote $X(\gamma_{21}(t)) = X_{21}^\theta(t) \partial_\theta|_{\gamma_{21}(t)} + X_{21}^\phi(t) \partial_\phi|_{\gamma_{21}(t)}$ to be the parallel transported vector field along $\gamma_{21}$, which satisfies 
$\frac{d X_{21}^\theta(t)}{d t}=0$ and $\frac{d X_{21}^\phi(t)}{d t}=-\tan (t) X^\phi_{21}$ with the initial value $X_{21}^\theta(0)=0$ and $X_{21}^\phi(0)=a$.
The solution is $X_{21}^\theta(t)= 0$ and $X_{21}^\phi(t)=\frac{a}{\cos(t)}$ for $t\in [0,\pi/2)$. Hence, the parallel transport of $u_0$ to $z=\gamma_{21}(\pi/4)$ is $\sqrt{2}a\partial_\phi|_z$. Continue with $\gamma_{22}(t)$ to parallel transport $\sqrt{2}a\partial_\phi|_z$ to $y$. Note that for $\gamma_{22}(t)$, by a calculation, we know $\gamma_{22}(t)=\varphi(\theta(t),\phi(t))$, where $\theta(t)=\cos^{-1}\left(\frac{\cos t}{\sqrt{2}}\right)$ and $\phi(t)=\tan^{-1}\left(\sqrt{2}\tan t\right)$. We can then express $\gamma_{22}'(t)=\theta'(t)\,\partial_\theta|_{\gamma_{22}(t)}+\phi'(t)\,\partial_\phi|_{\gamma_{22}(t)}$, where $\theta'(t)=\frac{\sin t}{\sqrt{2-\cos^2 (t)}}$ and $\phi'(t)=\frac{\sqrt{2}}{1+\sin^2 t}$, and hence $\dot{\gamma}_{22}^\theta(t)=\frac{\sin t}{\sqrt{2-\cos^2 (t)}}$ and $\dot{\gamma}_{22}^\phi(t)=\frac{\sqrt{2}}{1+\sin^2 t}$. On the other hand, $\Gamma^\theta_{\phi,\phi}(\gamma_{22}(t))=\frac{-1}{2}\cos(t)\sqrt{2-\cos^2(t)}$ and $\Gamma^\phi_{\phi,\theta}(\gamma_{22}(t))=\frac{\cos(t)}{\sqrt{2-\cos^2(t)}}$.
Denote $X(\gamma_{22}(t)) = X_{22}^\theta(t) \partial_\theta|_{\gamma_{22}(t)} + X_{22}^\phi(t) \partial_\phi|_{\gamma_{22}(t)}$ to be the parallel transported vector field along $\gamma_{22}$, which satisfies 
$\frac{dX_{22}^\theta(t)}{dt}-\frac{\cos(t)}{\sqrt{4-2\cos^2(t)}}X_{22}^\phi(t)=0$ and $\frac{dX_{22}^\phi(t)}{dt}+\frac{\sqrt{2}\cos(t)}{(2-\cos^2(t))^{3/2}}X_{22}^\theta(t)+\frac{\sin(t)\cos(t)}{2-\cos^2(t)}X_{22}^\phi(t)=0$ with the initial value $X_{22}^\theta(0)=0$ and $X_{22}^\phi(0)=\sqrt{2}a$. As a result, the solution is 
\begin{equation*}
X_{22}^\theta(t)= \frac{a\sin t}{\sqrt{2-\cos^2 (t)}}\,\quad X_{22}^\phi(t)=\frac{\sqrt{2}a}{1+\sin^2 t}\,,
\end{equation*}

where $t\in [0,\pi/2]$. Hence, the parallel transport of $\sqrt{2}a\partial_\phi|_z$ from $z$ to $y$ is $u_2:=\frac{\sqrt{2}a}{2}\partial_\phi|_y+\frac{\sqrt{2}a}{2}\partial_\theta|_y:=X_2$. 
Thus, we find that $\parallelslant_x^yu_0=u_1 \neq u_2=\parallelslant_z^y\parallelslant_x^zu_0$.
\end{example}

\section{Proof of Main Theorems}\label{sec:proof}
We prove Theorem \ref{thm:bias} and \ref{thm:var} in this section.

\subsection{Some background material for the proof}
We collect some lemmas that are needed for the proof in this subsection.
First, the following is the standard geometric expansions that will be used repeatedly in the proofs.

\begin{lemma}[Lemma B.6 in \cite{singer2011}]
In the normal coordinate around $x \in \mathcal{M}$, when $\|v\|_{g} \ll 1, v \in T_x \mathcal{M}$, the Riemannian measure satisfies
\begin{equation*}
\mathrm{d} V\left(\exp _x v\right)=\left(1-\frac{1}{6} \operatorname{Ric}_{k l}(x) v^k v^l+O\left(\|v\|^3\right)\right) \mathrm{d} v^1 \wedge \mathrm{d} v^2 \ldots \wedge \mathrm{d} v^d
\end{equation*}
where $\operatorname{Ric}$ is the Ricci curvature associated with riemannian metric $g$.
\end{lemma}

We next relate the ambient Euclidean distance in $\mathbb{R}^p$ to intrinsic distances on $\mathcal{M}$.
\begin{lemma}[\cite{singer2011} Lemma B.7 and \cite{talmon2019} Lemma A.2]
\label{lem:same_metric}
Fix $x \in \mathcal{M}$ and $y=\exp _x(v)$, where $v \in T_x \mathcal{M}$ with $\|v\|_{g} \ll 1$. We have
\begin{equation*}
\iota(y)=\iota(x)+\mathrm{d} \iota(v)+Q_2(v)+Q_3(v)+O\left(\|v\|_{g}^4\right),
\end{equation*}
where $\Pi$ is the second fundamental form of $\iota, Q_2(v):=\frac{1}{2} \Pi(v, v)$ and $Q_3(v):=\frac{2}{6} \nabla_v \Pi(v, v)$.

Further, for $z=\exp _x(u)$, where $u \in T_x \mathcal{M}$ with $\|u\|_{g} \ll 1$, we have
\begin{equation*}
\begin{aligned}
\left\|\iota(z)-\iota(y)\right\|= & \|u-v\|_{g}+\tilde{Q}_2(u,v)+\tilde{Q}_3(u,v)+O\left(\|u\|_{g}^4,\|v\|_{g}^4\right)
\end{aligned}
\end{equation*}
where
\begin{align*}
\tilde{Q}_2(u,v)&=\frac{\left\|Q_2(u)-Q_2(v)\right\|^2}{2\|u-v\|_{g}} \\
\tilde{Q}_3(u,v)&=\frac{\left\langle\mathrm{d} \iota(u-v), Q_3(u)-Q_3(v)\right\rangle}{\|u-v\|_{g}}\,.
\end{align*}
\end{lemma}

With above notations, the following lemma states the the asymptotical behavior of the landmark kernel. Please refer to Equation (46) in \cite{shen2022} or Lemma A.6 and Theorem 3.2 \cite{talmon2019} for more detail.
\begin{lemma}[Lemma SM 1.5 in \cite{yeh2024landmark}]\label{lem:refkernel}
Take $F \in C^3(\mathcal{M})$, $0<\gamma<1 / 2$ and $x, y \in \mathcal{M}$ so that $y=\exp _x v$, where $v \in T_x \mathcal{M}$ and $\|v\|_{g} \leq 2 \epsilon^\gamma$. Then, when $\epsilon$ is sufficiently small, the following holds:
\begin{equation*}
\begin{aligned}
& \int_{\mathcal{M}} K_\epsilon(x, z) K_\epsilon\left(z, y\right) F(z) \mathrm{d} V(z) \\
= &\, \epsilon^{d / 2}\big[F(x) A_{0,\epsilon}(v)+\epsilon^{1/2} A_{1,\epsilon}(F, v)+\epsilon A_{2,\epsilon}(F, v)\big]+O(\epsilon^{d/2+3 / 2})\,,
\end{aligned}
\end{equation*}
where
\begin{align*}
A_{0,\epsilon}(v)&\,:=\int_{\mathbb{R}^d} K\left(\left\| w\right\|\right) K(\|w-v / \sqrt{\epsilon}\|) \mathrm{d} w\,,\\
A_{1, \epsilon}(F, v)&\,:=\int_{\mathbb{R}^d} K\left(\left\| w\right\|\right) K(\|w-v / \sqrt{\epsilon}\|) \nabla_w F(x) \mathrm{d} w\,,\\
A_{2, \epsilon}(F, v)&\,:=F(x)B_{21,\epsilon}(v)+F(x)B_{20,\epsilon}(v)+B_{22,\epsilon}(F, v)\nonumber\,,
\end{align*}
where
\begin{align*}
B_{21,\epsilon}\,(v):= &\int_{\mathbb{R}^d}\big[K\big]^{\prime}\left(\left\| w\right\|\right) K(\|w-v / \sqrt{\epsilon}\|) \tilde{Q}_3(v) \mathrm{d} w \\
 &\hspace*{0.5cm}+\int_{\mathbb{R}^d} K\left(\left\| w\right\|\right)\big[K\big]^{\prime}(\|w-v / \sqrt{\epsilon}\|) \tilde{Q}_3(w, v / \sqrt{\epsilon}) \mathrm{d} w \nonumber\\
B_{20,\epsilon}(v):= &\int_{\mathbb{R}^d} K\left(\left\| w\right\|\right) K(\|w-v / \sqrt{\epsilon}\|) \operatorname{Ric}_{i j}(x) w^i w^j \mathrm{~d} w \,,\\
B_{22,\epsilon}(F, v):=&\int_{\mathbb{R}^d} K\left(\left\| w\right\|\right) K(\|w-v / \sqrt{\epsilon}\|) \frac{\nabla_{w, w}^2 F(x)}{2} \mathrm{d} w \,,
\end{align*}
where $\operatorname{Ric}$ is the Ricci curvature and $\Pi$ is the second fundamental associated with $g$. Note that $B_{21,\epsilon}(v)$ depends on the first derivative of kernel functions, $B_{20,\epsilon}(v)$ depends on Ricci curvature, and $B_{22,\epsilon}(F,v)$ depends on the second derivative of $F$. Moreover, $A_{0,\epsilon}(v)=A_{0,\epsilon}(-v)$, $A_{1,\epsilon}(F,v)=-A_{1,\epsilon}(F,-v)$ and $A_{2,\epsilon}(F,v)=A_{2,\epsilon}(F,-v)$.
\end{lemma}

We remark that if the chosen kernels are both Gaussian, the leading term of effective kernel defined in Proposition \ref{prop:eff} is Gaussian. Specifically, when kernel $K$ is normalized Gaussian ($\mu_{1,0,1}=1$), that is, $K(t)=e^{-t^2} / \sqrt{\pi}$, we have
\begin{equation*}
A_0(v)=(\pi/2)^{d/2} e^{-\frac{\left\|v\right\|^2 }{ 2\epsilon}},
\end{equation*}
which satisfies the exponential decay property of the kernel functions.

Finally, we state the truncation lemma \cite{coifman2006,singer2015} in bundle version. For convenience, denote
\begin{equation*}
\tilde{B}_h(x):=\exp _x\left(B_h\right),
\end{equation*}
where $B_h=\left\{u \in T_x \mathcal{M} \mid\|u\|_{g} \leq h\right\} \subset T_x \mathcal{M}$ is a $d$-dim disk with the center 0 and the radius $h>0$.
\begin{lemma}[Lemma B.1. in \cite{singer2015}]\label{lem:trunc}
Suppose $X \in$ $L^{\infty}(\mathcal{E})$ and $0<\gamma<1 / 2$. Then, when $\epsilon$ is small enough, for all $x \in \mathrm{M}$ the following holds:
\begin{equation*}
\left|\int_{\mathcal{M} \backslash \widetilde{B}_{\epsilon }^{\gamma}(x)} \epsilon^{-d / 2} K_\epsilon(x, y)\parallelslant_y^x X(y) \mathrm{d} V(y)\right|=O\left(\epsilon^2\right),
\end{equation*}
where $O\left(\epsilon^2\right)$ depends on $\|X\|_{L^{\infty}}$.
\end{lemma}

In addition, to approximate the parallel transport on the bundle, we recall the following lemma, where the parallel transport is derived from the connection on the tangent bundle. To distinguish it from the parallel transport generated by the connection in the principal bundle defined earlier, we denote parallel transport derived from the connection on the tangent bundle from $y$ to $x$ as $P_y^x$. More information can be found in \cite[Theorem 3]{pennec2019curvature}.

\begin{lemma}\cite[Equation (3)]{gavrilov2007double}\label{lem:double_exp}
Let $(\mathcal M,g)$ be a Riemannian manifold and fix $x\in\mathcal M$.
Let $u,v\in T_x\mathcal M$ be sufficiently small and set
\[
z:=\exp_x(u),\qquad y:=\exp_x(v).
\]
Let $w\in T_z\mathcal M$ be the unique vector (for $\|u\|,\|v\|$ small) such that
\[
y=\exp_z(w).
\]
Then the parallel transport of $w$ from $z$ back to $x$ along the (minimizing) geodesic
$\gamma(t):=\exp_x(tv)$ satisfies
\[
P_z^x w
=
v-u+\frac{1}{6}R(u,v)(u-2v)
+
\mathcal{O}\bigl((\|u\|+\|v\|)^4\bigr),
\]
where $R$ is the curvature tensor (with the convention used in \cite{gavrilov2007double}),
and the implied constant depends on bounds for the first covariant derivative of $R$
in a neighborhood of $x$.
\end{lemma}

\begin{proof}
We follow the argument of \cite[Theorem~3]{pennec2019curvature}, starting from the two-step exponential expansion
\cite[Eq.\ (3)]{gavrilov2007double}. For $a,b\in T_x\mathcal M$ sufficiently small one has
\begin{equation}\label{eq:double-exp}
\exp^{-1}_x\!\Bigl(\exp_{\exp_x(a)}\bigl(P_x^{\exp_x(a)} b\bigr)\Bigr)
=
a+b+\frac{1}{6}R(b,a)a+\frac{1}{3}R(b,a)b
+
\mathcal{O}\bigl((\|a\|+\|b\|)^4\bigr),
\end{equation}
where $P_x^{\exp_x(a)}:T_x\mathcal M\to T_{\exp_x(a)}\mathcal M$ denotes parallel transport along the radial
geodesic $t\mapsto \exp_x(ta)$.

We apply \eqref{eq:double-exp} with $a=u$ and choose $b\in T_x\mathcal M$ so that the left-hand side equals $v$.
Let $z=\exp_x(u)$ and $y=\exp_x(v)$, and define
\[
b:=P_{z}^{x}\,\log_{z}(y)=P_z^x w,
\]
where $w:=\log_z(y)\in T_z\mathcal M
\text{ so that }y=\exp_z(w)$. Then $P_x^{z}b=w$ and hence $\exp_{\exp_x(u)}(P_x^{\exp_x(u)}b)=\exp_z(w)=y$, so indeed the left-hand side of
\eqref{eq:double-exp} is $\log_x(y)=v$. Therefore \eqref{eq:double-exp} yields the implicit relation
\begin{equation}\label{eq:implicit-b}
v
=
u+b+\frac{1}{6}R(b,u)u+\frac{1}{3}R(b,u)b
+
\mathcal{O}\bigl((\|u\|+\|b\|)^4\bigr).
\end{equation}
In particular, ignoring the curvature terms shows $b=v-u+\mathcal{O}((\|u\|+\|v\|)^2)$, so $\|b\|=O(\|u\|+\|v\|)$
and the remainder in \eqref{eq:implicit-b} may be written as $\mathcal{O}((\|u\|+\|v\|)^4)$.

Write $b=(v-u)+r_3$, where $r_3=\mathcal{O}((\|u\|+\|v\|)^3)$. Substituting into \eqref{eq:implicit-b} and canceling
the linear terms $u+(v-u)$ against $v$, we obtain
\begin{equation}\label{eq:r3-eq}
0
=
r_3+\frac{1}{6}R(b,u)u+\frac{1}{3}R(b,u)b
+\mathcal{O}\bigl((\|u\|+\|v\|)^4\bigr).
\end{equation}
Since $r_3$ is already cubic, replacing $b$ by $(v-u)$ inside the curvature terms changes the right-hand side by
$\mathcal{O}((\|u\|+\|v\|)^4)$; hence \eqref{eq:r3-eq} implies
\[
r_3
=
-\frac{1}{6}R(v-u,u)u-\frac{1}{3}R(v-u,u)(v-u)
+\mathcal{O}\bigl((\|u\|+\|v\|)^4\bigr).
\]
Using bilinearity in the first two slots and $R(u,u)=0$, we have $R(v-u,u)=R(v,u)$, and therefore
\[
r_3
=
-\frac{1}{6}R(v,u)u-\frac{1}{3}R(v,u)(v-u)
=
\frac{1}{6}R(v,u)u-\frac{1}{3}R(v,u)v
+\mathcal{O}\bigl((\|u\|+\|v\|)^4\bigr).
\]
Consequently,
\[
b
=
v-u+r_3
=
v-u+\frac{1}{6}R(v,u)u-\frac{1}{3}R(v,u)v
+\mathcal{O}\bigl((\|u\|+\|v\|)^4\bigr).
\]
Recalling that $b=P_z^x w$, this gives the claimed expansion. Finally, the cubic term can be rewritten as
$\frac{1}{6}R(u,v)(u-2v)$ by the curvature convention adopted in \cite{gavrilov2007double}. Therefore, the desired results $P_z^xw=\ell(u,v)=v-u+r_3(u,v)$ follows.
\end{proof}

\subsection{Proof of Thoerem \ref{thm:bias}}

First, we approximate the parallel transport in a sufficiently small neighborhood. We restate
Lemma~\ref{lem:double_para} for convenience and include a self-contained proof.

\begin{lemma}
Let $X\in C^{3}(\E)$ be a $C^{3}$ section of the vector bundle $\E\to\M$ endowed with a metric connection,
and let $x,y,z\in\M$ be such that $x,y,z$ lie in each other's $\sqrt{\epsilon}$-neighborhoods. Then
\[
\parallelslant_{z}^{x}\,\parallelslant_{y}^{z}X(y)
=
\parallelslant_{y}^{x}X(y)
+\mathcal{O}(\epsilon^{3/2}),
\]
where the implied constant depends only on uniform bounds for the curvature 
of the underlying connection and on the $C^{3}$-norm of $X$ on the relevant neighborhood.
\end{lemma}

\begin{proof}
Work in a normal coordinate neighborhood about $x$ small enough so that all points under consideration are joined
by unique minimizing geodesics. Write $z=\exp_x(u)$ and $y=\exp_x(v)$ with $u,v\in T_x\M$ and
$\|u\|_g,\|v\|_g=O(\sqrt{\epsilon})$. Let $w\in T_z\M$ be defined by $y=\exp_z(w)$, so $\|w\|_g=O(\sqrt{\epsilon})$
as well. We compare the two ways of transporting $X(y)\in \E_y$ to the fiber $\E_x$.

Let $\gamma(t):=\exp_x(tv)$, $t\in[0,1]$, be the geodesic from $x$ to $y$. Parallel transport $w$ from $z$ to $x$
along the geodesic from $z$ to $x$ (equivalently, apply $P_z^x$ on $T\M$) and denote
\[
\widetilde w(x):=P_z^x w\in T_x\M.
\]
By Lemma~\ref{lem:double_exp} (applied on the tangent bundle), we have the third-order expansion
\begin{equation}\label{eq:w-approx}
\widetilde w(x)=v-u+\mathcal{O}\bigl((\|u\|_g+\|v\|_g)^3\bigr)
=
v-u+\mathcal{O}(\epsilon^{3/2}),
\end{equation}
where the implied constant depends on curvature bounds.

Next we use covariant Taylor expansions for sections. Since $X\in C^3(\E)$, the section
$Z\mapsto \parallelslant_Z^{z}X(Z)\in \E_z$ admits a second-order Taylor expansion at $z$ along the geodesic
$s\mapsto \exp_z(sw)$:
\begin{equation}\label{eq:taylor-at-z}
\parallelslant_{y}^{z}X(y)
=
X(z)+\nabla^{\E}_{w}X(z)+\frac12\,\nabla^{\E\,2}_{w,w}X(z)+\mathcal{O}(\|w\|_g^{3}).
\end{equation}
Applying $\parallelslant_z^x$ and using that parallel transport commutes with the covariant derivative along a
geodesic (equivalently, $\parallelslant_z^x(\nabla^{\E}_{w}X(z))=\nabla^{\E}_{P_z^x w}\bigl(\parallelslant_z^x X(z)\bigr)$
and similarly at second order), we obtain
\begin{equation}\label{eq:transport-expand}
\parallelslant_z^x\parallelslant_y^zX(y)
=
\parallelslant_z^xX(z)
+\nabla^{\E}_{\widetilde w(x)}\bigl(\parallelslant_z^xX(z)\bigr)
+\frac12\,\nabla^{\E\,2}_{\widetilde w(x),\widetilde w(x)}\bigl(\parallelslant_z^xX(z)\bigr)
+\mathcal{O}(\|w\|_g^{3}).
\end{equation}
Now expand the intermediate term $\parallelslant_z^xX(z)$ at $x$ along the geodesic $t\mapsto \exp_x(tu)$:
\begin{equation}\label{eq:taylor-at-x}
\parallelslant_z^xX(z)
=
X(x)+\nabla^{\E}_{u}X(x)+\frac12\,\nabla^{\E\,2}_{u,u}X(x)+\mathcal{O}(\|u\|_g^{3}).
\end{equation}
Substituting \eqref{eq:taylor-at-x} into \eqref{eq:transport-expand} and keeping all terms up to second order in
$u$ and $\widetilde w(x)$ yields
\begin{align*}
\parallelslant_z^x\parallelslant_y^zX(y)
&=
X(x)
+\nabla^{\E}_{u}X(x)
+\nabla^{\E}_{\widetilde w(x)}X(x)
+\frac12\,\nabla^{\E\,2}_{u,u}X(x)
+\nabla^{\E\,2}_{u,\widetilde w(x)}X(x)
+\frac12\,\nabla^{\E\,2}_{\widetilde w(x),\widetilde w(x)}X(x)  \\
&\qquad\qquad
+\mathcal{O}\!\left(\|u\|_g^{3}+\|\widetilde w(x)\|_g^{3}+\|w\|_g^{3}\right).
\end{align*}
Using \eqref{eq:w-approx} and the bilinearity of $\nabla^{\E\,2}$ in its directional arguments,
\[
u+\widetilde w(x)=u+(v-u)+\mathcal{O}(\epsilon^{3/2})=v+\mathcal{O}(\epsilon^{3/2}),
\]
and therefore the quadratic part above is exactly $\frac12\,\nabla^{\E\,2}_{v,v}X(x)$ up to an error
$\mathcal{O}(\epsilon^{3/2})$, while the linear part is $\nabla^{\E}_{v}X(x)$ up to the same order.
Since $\|u\|_g,\|v\|_g,\|w\|_g=O(\sqrt{\epsilon})$, the remainder term is $\mathcal{O}(\epsilon^{3/2})$.
Consequently,
\begin{equation}\label{eq:two-step-expansion}
\parallelslant_z^x\parallelslant_y^zX(y)
=
X(x)+\nabla^{\E}_{v}X(x)+\frac12\,\nabla^{\E\,2}_{v,v}X(x)+\mathcal{O}(\epsilon^{3/2}).
\end{equation}
On the other hand, applying the same covariant Taylor expansion directly along the geodesic $\gamma(t)=\exp_x(tv)$ gives
\[
\parallelslant_y^xX(y)
=
X(x)+\nabla^{\E}_{v}X(x)+\frac12\,\nabla^{\E\,2}_{v,v}X(x)+\mathcal{O}(\|v\|_g^{3})
=
X(x)+\nabla^{\E}_{v}X(x)+\frac12\,\nabla^{\E\,2}_{v,v}X(x)+\mathcal{O}(\epsilon^{3/2}).
\]
Comparing this with \eqref{eq:two-step-expansion} proves the claim.
\end{proof}

The following lemma generalizes Lemma SM 1.6 in \cite{yeh2024landmark}, extending the discussion from trivial bundles to principal bundles.
\begin{lemma}\label{lem:double_kernel}
Fix $x \in \mathcal{M}$ and pick $F,G \in C^3(\mathcal{M})$ and $X\in C^3(\E)$. Then, when $\epsilon$ is sufficiently small, we have
\begin{align*}
&\int_{\mathcal{M}} \int_{\mathcal{M}} K_\epsilon(x, z) K_\epsilon\left(z, y\right) F(z) \mathrm{d} V(z) G\left(y\right) \parallelslant_y^x X(y)d V\left(y\right)\\
=&F(x) G(x) X(x)+\epsilon F(x)G(x)X(x)W(x)+\epsilon \frac{\mu_{2,0,1}}{d} F(x) G(x) \nabla^2X(x) \\
&+\epsilon \frac{\mu_{2,0,1}}{d} F(x) X(x) \Delta G(x)+\epsilon\frac{\mu_{2,0,1}}{2 d} G(x) X(x) \Delta F(x)\\
& +\epsilon \frac{2\mu_{2,0,1}}{d} F(x) \nabla^{\E} X(x) \cdot \nabla G(x)+\epsilon \frac{\mu_{2,0,1}}{d} G(x) \nabla^{\E} X(x) \cdot \nabla F(x) \\
& +\epsilon \frac{\mu_{2,0,1}}{d} X(x) \nabla F(x) \cdot \nabla G(x)_\mathcal{O}(\epsilon^{3/2})
\end{align*}
where
\begin{align*}
W(x)=\frac{3\mu_{2,0,1}}{d}s(x)+\int_{\mathbb{R}^d}B_{21,\epsilon}(\sqrt{\epsilon}v)\mathrm{d}v\,,
\end{align*}
where $s$ is the scalar curvature associated with $g$ and $B_{21,\epsilon}$ is defined in Lemma \ref{lem:refkernel}. Thus, $W$ depends on the scalar curvatures and the second fundamental form.
\end{lemma}

\begin{proof}
This proof is very similar to Lemma SM 1.6 in \cite{yeh2024landmark}, with the only difference being the inclusion of the bundle field $X\in C^3(\E)$. Consequently, this proof will provide a brief overview of the main arguments, while the detailed calculations can be referenced in Lemma SM 1.6 in \cite{yeh2024landmark} or Lemma A.6 in \cite{talmon2019} for further information.

By Lemma \ref{lem:trunc} and Lemma \ref{lem:refkernel}, let $y=\exp_xv$, where $v\in T_x\M$ and we directly compute
\begin{align*}
&\int_{\mathcal{M}} \int_{\mathcal{M}} K_\epsilon(x, z) K_\epsilon\left(z, y\right) F(z) \mathrm{d} V(z) G\left(y\right)\parallelslant_y^x X(y) \mathrm{d} V\left(y\right)\\
=& \epsilon^{d / 2}\int_{\mathbb{R}^d} \left[F(x) A_{0,\epsilon}(v)+\epsilon^{1/2} A_{1,\epsilon}(F, v)+\epsilon A_{2,\epsilon}(F, v)+\mathcal{O}(\epsilon^{3 / 2})\right] \\
& \times\left[X(x)+\nabla_v^\E X(x)+\frac{\nabla_{vv}^{\E} X(x)}{2}+\mathcal{O}\left(\|v\|^3\right)\right] \\
& \times\left[G(x)+\nabla_v G(x)+\frac{\nabla_{vv} G(x)}{2}+\mathcal{O}\left(\|v\|^3\right)\right] \\
& \times\left[1-\operatorname{Ric}_{i j}(x) v^i v^j+\mathcal{O}\left(\|v\|^3\right)\right] \mathrm{d} v+\mathcal{O}\left(\epsilon^{d +3/2}\right) \\
:=\,&\epsilon^{d/2}[I_0+I_{11}+ \epsilon^{1/2}I_{12}+ \epsilon^{1/2}I_{13}+I_{21}+I_{22}+\epsilon I_{23}+ I_{24}]+\mathcal{O}(\epsilon^{d+3/2})\,,
\end{align*}
where
\begin{align*}
I_0 &=F(x)G(x)X(x)\int_{\mathbb{R}^d} A_{0,\epsilon}(v)dv\\
I_{11} &=F(x)\int_{\mathbb{R}^d} A_{0,\epsilon}(v)\nabla_v^{\E} X\cdot\nabla_v G(x)dv,\\ I_{12}& =G(x)\int_{\mathbb{R}^d} A_{1,\epsilon}(F, v)\nabla_v^{\E}X(x)dv,\\
I_{13} &=X(x)\int_{\mathbb{R}^d} A_{1,\epsilon}(F, v)\nabla_vG(x)dv
\end{align*}
and
\begin{align*}
I_{21}&=F(x)G(x)\int_{\mathbb{R}^d} A_{0,\epsilon}(v)\frac{\nabla_{vv}^{\E}X(x)}{2}dv, \quad I_{22}=F(x)X(x)\int_{\mathbb{R}^d} A_{0,\epsilon}(v)\frac{\nabla_{vv}G(x)}{2}dv,\\
I_{23} &=G(x)X(x)\int_{\mathbb{R}^d} A_{2,\epsilon}(F,v)dv,\quad I_{24} = F(x)G(x)X(x)\int_{\mathbb{R}^d} A_{0,\epsilon}(v)\operatorname{Ric}_{i j}(x) v^i v^jdv\,.
\end{align*}

Based on the straightforward computation and $\mu_{1,0,1}=1$, with particular attention to the differentiation of the bundle, we can obtain the following:
\begin{align*}
I_0 &=\epsilon^{d/2}F(x)G(x)X(x)\\
I_{11} &=\epsilon^{d/2+1}F(x)\frac{2\mu_{2,0,1}}{d}\nabla^{\E}X(x)\cdot\nabla G(x),\\
I_{12}& =\epsilon^{d/2+1/2}G(x)\frac{\mu_{2,0,1}}{d}\nabla^{\E}X(x)\cdot\nabla F(x),\\
I_{13} &=\epsilon^{d/2+1/2}X(x)\frac{\mu_{2,0,1}}{d}\nabla F(x)\cdot\nabla G(x)
\end{align*}
and
\begin{align*}
I_{21}&=\epsilon^{d/2+1}F(x)G(x)\frac{\mu_{2,0,1}}{d}\nabla^2 X(x), \quad I_{22}=\epsilon^{d/2+1}F(x)X(x)\frac{\mu_{2,0,1}}{d}\Delta G(x),\\
I_{23} &=G(x)X(x)\frac{\mu_{2,0,1}}{2d}\Delta F(x)+\epsilon^{d/2+1}\frac{2\mu_{2,0,1}}{d}F(x)G(x)X(x)s(x),\\
&\quad\quad+\epsilon^{d/2+1}F(x)G(x)X(x)\int_{\mathbb{R}^d}B_{21,\epsilon}(\sqrt{\epsilon}v)\mathrm{d}v\\
I_{24} &= \epsilon^{d/2+1}\frac{\mu_{2,0,1}}{d}F(x)G(x)X(x)s(x)\,.
\end{align*}
where $s$ is the scalar curvature associated with $g$ and $B_{21,\epsilon}$ is defined in Lemma \ref{lem:refkernel}. Remark that the function $B_{21,\epsilon}$ depends on the second fundamental form. For convenience, we denote
\begin{equation*}
W(x)=\frac{3\mu_{2,0,1}}{d}s(x)+\int_{\mathbb{R}^d}B_{21,\epsilon}(\sqrt{\epsilon}v)\mathrm{d}v\,,
\end{equation*}
Putting everything together, the desired result follows.
\end{proof}

With above lemmas, we prove Theorem \ref{thm:bias}. First, we approximate the piecewise parallel transport based on Lemma \ref{lem:double_para}, and then we integrate the numerator and denominator separately according to Lemma \ref{lem:double_kernel}. The proof is as follows.
\begin{proof}
Let $X\in C(\E)$. By definition of landmark $(\beta,\alpha)$-normalized diffusion operator, we have
\begin{equation*}
T_{\epsilon,\beta,\alpha}^{(R)}X(x)=\int_{\mathcal{M}} \frac{S_{\epsilon,\beta,\alpha}^{(R)}(x, y)}{d_{\epsilon,\beta,\alpha}^{(R)}(x)}  X(y) d \nu(y)\,.
\end{equation*}
We first compute the numerator. By Lemma \ref{lem:double_para}, we have
\begin{align}
&d_{\exnormd,\beta}(x)^{\alpha}\epsilon^{-\alpha\beta d+\alpha d}\int_{\mathcal{M}} S_{\epsilon,\beta,\alpha}^{(R)}(x, y)  X(y) d \nu(y)\notag\\
&\quad=\epsilon^{-\alpha\beta d+\alpha d}\int_{\M}\int_{\mathcal{M}} K_\epsilon(x,z) K_\epsilon\left(z,y\right)d_{\innormd}(z)^{-\beta}\parallelslant^x_z\parallelslant^z_y d \nu^{\Z}(z)X(y)d_{\exnormd,\beta}(y)^{-\alpha}d\nu(y)\notag\\
&\quad=\epsilon^{-\alpha\beta d+\alpha d}\int_{\M}\int_{\mathcal{M}} K_\epsilon(x,z) K_\epsilon\left(z,y\right)d_{\innormd}(z)^{-\beta} d \nu^{\Z}(z)\parallelslant^x_yX(y)d_{\exnormd,\beta}(y)^{-\alpha}d\nu(y)+\mathcal{O}(\epsilon^{3/2})\notag\\
&\quad=\int_{\M}K^{(R)}_{\epsilon,\beta}(x,y)\parallelslant^x_yX(y)\frac{p(y)}{q_\beta(y)^{\alpha}p(y)^{\alpha}(1+\epsilon\bar{E}_2(y)+\mathcal{O}(\epsilon^2))^{\alpha}}dV(y)+\mathcal{O}(\epsilon^{3/2})\notag\\
&\quad=\int_{\M}K^{(R)}_{\epsilon,\beta}(x,y)\parallelslant^x_yX(y)\rho_{\beta,\alpha}(y)dV(y)\label{eq:bias_up}\\
&\quad\quad\quad-\epsilon\alpha\int_{\M}K^{(R)}_{\epsilon,\beta}(x,y)\parallelslant_y^x X(y)\rho_{\beta,\alpha}(y)\bar{E}_2(y)dV(y)+\mathcal{O}(\epsilon^{3/2})\label{eq:bias_down}
\end{align}
where $\rho_{\beta,\alpha}(x)=q_\beta(x)^{-\alpha}p(x)^{1-\alpha}$. Now, focus on (\ref{eq:bias_up}) with $\epsilon^{\beta d}$,
\begin{align*}
&\epsilon^{\beta d}\int_{\M}K^{(R)}_{\epsilon,\beta}(x,y)\parallelslant^x_yX(y)\rho_{\beta,\alpha}(y)dV(y)\\
&\quad=\int_{\M}\int_{\mathcal{M}} K_\epsilon(x,z) K_\epsilon\left(z,y\right)\frac{p_\Z(z)}{p(z)^\beta p_\Z(z)^\beta(1+\epsilon \bar{E}_1(z)+\mathcal{O}(\epsilon^2))^\beta} d V(z)\\
&\hspace{2.5cm}\parallelslant_y^xX(y)\rho_{\beta,\alpha}(y)dV(y)\\
&\quad=\int_{\M}\int_{\mathcal{M}} K_\epsilon(x,z) K_\epsilon\left(z,y\right)q_\beta(z)dV(z)\parallelslant_y^xX(y)\rho_{\beta,\alpha}(y)dV(y)\\
&\quad\quad-\epsilon\beta\int_{\M}\int_{\mathcal{M}} K_\epsilon(x,z) K_\epsilon\left(z,y\right)q_\beta(z)\bar{E}_1(z)dV(z)\parallelslant_y^xX(y)\rho_{\beta,\alpha}(y)dV(y)+\mathcal{O}(\epsilon^2)
\end{align*}
where $q_\beta(x)=p(x)^{-\beta}p_{\Z}(x)^{1-\beta}$. By the similar argument, (\ref{eq:bias_down}) can be expressed as follows:
\begin{align*}
&\epsilon\alpha\int_{\M}K^{(R)}_{\epsilon,\beta}(x,y)\parallelslant_y^xX(y)\rho_{\beta,\alpha}(y)\bar{E}_2(y)dV(y)\\
&\quad=\epsilon\alpha\int_{\M}\int_{\mathcal{M}} K_\epsilon(x,z) K_\epsilon\left(z,y\right)q_\beta(z)dV(z)\parallelslant_y^xX(y)\rho_{\beta,\alpha}(y)\bar{E}_2(y)dV(y)+\mathcal{O}(\epsilon^2)
\end{align*}
It is sufficient to apply Lemma \ref{lem:double_kernel} on the expressions of (\ref{eq:bias_up}) and (\ref{eq:bias_down}),
\begin{equation}\label{eq:nume}
\begin{aligned}
&d_{\exnormd,\beta}(x)^{\alpha}\epsilon^{-\alpha\beta d+\alpha d+\beta d}\int_{\mathcal{M}} S_{\epsilon,\beta,\alpha}^{(R)}(x, y)  X(y) d \nu(y)\\
&\quad=q_{\beta}(x) \rho_{\beta,\alpha}(x) X(x)+\epsilon q_{\beta}(x)\rho_{\beta,\alpha}(x)X(x)W(x)+\epsilon \frac{\mu_{2,0,1}}{d} q_{\beta}(x) \rho_{\beta,\alpha}(x) \nabla^2X(x) \\
&\quad\quad+\epsilon \frac{\mu_{2,0,1}}{d} q_{\beta}(x) X(x) \Delta \rho_{\beta,\alpha}(x)+\epsilon\frac{\mu_{2,0,1}}{2 d} \rho_{\beta,\alpha}(x) X(x) \Delta q_{\beta}(x)\\
&\quad\quad+\epsilon \frac{2\mu_{2,0,1}}{d} q_{\beta}(x) \nabla^{\E} X(x) \cdot \nabla \rho_{\beta,\alpha}(x)+\epsilon \frac{\mu_{2,0,1}}{d} \rho_{\beta,\alpha}(x) \nabla^{\E} X(x) \cdot \nabla q_{\beta}(x) \\
&\quad\quad+\epsilon \frac{\mu_{2,0,1}}{d} X(x) \nabla q_{\beta}(x) \cdot \nabla \rho_{\beta,\alpha}(x)\\
&\quad\quad-\epsilon\beta q_{\beta}(x) \rho_{\beta,\alpha}(x) X(x)\bar{E}_1(x)-\epsilon\alpha q_{\beta}(x) \rho_{\beta,\alpha}(x) X(x)\bar{E}_2(x)+\mathcal{O}(\epsilon^{3/2})
\end{aligned}
\end{equation}
where
\begin{align*}
W(x)=\frac{3\mu_{2,0,1}}{d}s(x)+\int_{\mathbb{R}^d}B_{21,\epsilon}(\sqrt{\epsilon}v)\mathrm{d}v\,.
\end{align*}
By the similar arguments, the denominator of $T_{\epsilon,\beta,\alpha}^{(R)}X(x)$ can be organized as
\begin{align*}
&d_{\exnormd,\beta}(x)^{\alpha}\epsilon^{-\alpha\beta d+\alpha d+\beta d}d_{\epsilon,\beta,\alpha}^{(R)}(x)\\
&\quad=\int_{\M}\int_{\mathcal{M}} K_\epsilon(x,z) K_\epsilon\left(z,y\right)d_{\innormd}(z)^{-\beta}d \nu^{\Z}(z)d_{\exnormd,\beta}(y)^{-\alpha}d\nu(y)\notag\\
&\quad=q_{\beta}(x) \rho_{\beta,\alpha}(x) +\epsilon q_{\beta}(x)\rho_{\beta,\alpha}(x)W(x)+\epsilon \frac{\mu_{2,0,1}}{d} \nabla q_{\beta}(x) \cdot \nabla \rho_{\beta,\alpha}(x)\\
&\quad\quad+\epsilon \frac{\mu_{2,0,1}}{d} q_{\beta}(x)  \Delta \rho_{\beta,\alpha}(x)+\epsilon\frac{\mu_{2,0,1}}{2 d} \rho_{\beta,\alpha}(x)  \Delta q_{\beta}(x)\\
&\quad\quad-\epsilon\beta q_{\beta}(x) \rho_{\beta,\alpha}(x)\bar{E}_1(x)-\epsilon\alpha q_{\beta}(x) \rho_{\beta,\alpha}(x)\bar{E}_2(x)+\mathcal{O}(\epsilon^{3/2})
\end{align*}
Let normal coordinates around $x$ associated with $g$ and denote $\{E_i\}_{i=1}^d \subset T_x \mathcal{M}$ to be an orthonormal basis associated with $g$. Putting everything together, the desired results follows,
\begin{align*}
T_{\epsilon,\beta,\alpha}^{(R)} X(x)=X(x)&+\frac{\epsilon \mu_{2,0,1}}{d} \nabla^{2} X(x)\\
&+\frac{\epsilon \mu_{2,0,1}}{d}\sum^{d}_{i=1}\left(\frac{2 \nabla_{E_i} \rho_{\beta,\alpha}(x)}{\rho_{\beta,\alpha}(x)}+\frac{\nabla_{E_i} q_\beta(x)}{q_\beta(x)}\right) \nabla^{\mathcal{E}}_{E_i} X(x)+\mathcal{O}\left(\epsilon^{3/2}\right)\,,
\end{align*}
where $q_\beta(x)=p(x)^{-\beta}p_{\Z}(x)^{1-\beta}$, $\rho_{\beta,\alpha}(x)=q_\beta(x)^{-\alpha}p(x)^{1-\alpha}$.
\end{proof}

\subsection{Proof of Thoerem \ref{thm:var}}
Before we prove Theorem \ref{thm:var}, we define some functions for convenience:
\begin{align*}
\hat{K}_{\epsilon,0}^{(R)}(z_k,z_l)&:=\frac{1}{n} \sum_{i=1}^n \epsilon^{-d / 2} K_\epsilon\left(z_k, x_i\right) K_\epsilon\left(x_i, z_l\right),\\
\hat{d}_{\innormd,\epsilon,m}(z_k)&:=\frac{1}{m} \sum_{l=1}^m \epsilon^{-d / 2} \hat{K}_{\epsilon,0}^{(R)}(z_k,z_l),\\
\hat{K}_{\epsilon,\beta,m}^{(R)}(x_i,x_j)&:=\frac{1}{m} \sum_{k=1}^m \epsilon^{-d / 2} K_\epsilon\left(x_i, z_k\right) K_\epsilon\left(z_k, x_j\right)\hat{d}_{\innormd,\epsilon,m}(z_k)^{-\beta},\\
\hat{d}_{\exnormd,\epsilon,\beta,n}(x_i)&:=\frac{1}{n} \sum_{j=1}^n \epsilon^{-d / 2} \hat{K}_{\epsilon,\beta,m}^{(R)}(x_i,x_j),\\
\hat{K}_{\epsilon,\beta,\alpha,n}^{(R)}(x_i,x_j)&:= \hat{K}_{\epsilon,\beta,m}^{(R)}\left(x_i, x_j\right) \hat{d}_{\exnormd,\epsilon,\beta,n}(x_j)^{-\alpha},\\
\hat{d}^{(R)}_{\epsilon,\beta,\alpha,n}(x_i)&:=\frac{1}{n} \sum_{j=1}^n \epsilon^{-d / 2} \hat{K}_{\epsilon,\beta,\alpha,m}^{(R)}(x_i,x_j)\,.
\end{align*}

\begin{proof}
Fix $x_i\in\M$ and $\alpha,\beta\in [0,1]$. Recall that $\Omega_{ik}=u^{-1}_i\parallelslant_k^iu_k$. Then, by Lemma \ref{lem:double_para}, we have
\begin{equation*}
\Omega_{ik}\Omega_{jk}^\top=u^{-1}_i\parallelslant_k^i\parallelslant_j^ku_j=u^{-1}_i\parallelslant_j^iu_j+\mathcal{O}(\epsilon^{3/2})=\Omega_{ij}+\mathcal{O}(\epsilon^{3/2})
\end{equation*}
For convenience, we define
\begin{equation*}
[L\boldsymbol{f}][i]:=\frac{1}{\epsilon}\left[\left(I_{nq}-\left(\mathbf{D}_{\beta,\alpha}^{(R)}\right)^{-1} \mathbf{S}_{\beta,\alpha}^{(R)}\right) \mathbf{X}\right][i]
\end{equation*}
By above statement,
\begin{align*}
\epsilon[L\boldsymbol{f}](i)&=\frac{\frac{1}{n}\sum_{j=1}^n \hat{K}_{\epsilon,\beta,\alpha,m}^{(R)}(x_i, x_j) (\mathbf{X}[i]-\Omega_{ij}\mathbf{X}[j]) }{\frac{1}{n}\sum_{j=1}^n \hat{K}_{\epsilon,\beta,\alpha,m}^{(R)}(x_i, x_j) }+\mathcal{O}(\epsilon^2)\\
&=\frac{\frac{1}{n}\sum_{j=1}^n \hat{K}_{\epsilon,\beta,m}^{(R)}(x_i, x_j)\hat{d}_{\exnormd,\epsilon,\beta,n}(x_j)^{-\alpha} (\mathbf{X}[i]-\Omega_{ij}\mathbf{X}[j]) }{\frac{1}{n}\sum_{j=1}^n \hat{K}_{\epsilon,\beta,m}^{(R)}(x_i, x_j)\hat{d}_{\exnormd,\epsilon,\beta,n}(x_j)^{-\alpha} }+\mathcal{O}(\epsilon^{3/2})
\end{align*}
By Lemma SM 1.8, SM 1.9 in \cite{yeh2024landmark} and the process in proof of Theorem 4.10 in \cite{yeh2024landmark}, we can obtain that with probability $1-\mathcal{O}(n^{-2})$,
\begin{align*}
\epsilon[L\boldsymbol{f}](i)=&\frac{\frac{1}{n}\sum_{j=1}^n K_{\epsilon,\beta,m}^{(R)}(x_i, x_j)\hat{d}_{\exnormd,\epsilon,\beta,n}(x_j)^{-\alpha} (\mathbf{X}[i]-\Omega_{ij}\mathbf{X}[j]) }{\frac{1}{n}\sum_{j=1}^n K_{\epsilon,\beta,m}^{(R)}(x_i, x_j)\hat{d}_{\exnormd,\epsilon,\beta,n}(x_j)^{-\alpha} }\\
&+\mathcal{O}(\epsilon^{3/2})+\mathcal{O}\left(\frac{\sqrt{\log(n)}}{n^{\gamma/2}\epsilon^{d/4}}\right)
\end{align*}
By proof of Theorem 4.10 in \cite{yeh2024landmark}, we can obtain that with probability $1-\mathcal{O}(n^{-2})$,
\begin{equation*}
\epsilon^{-d(1-\beta)}d_{\exnormd,\epsilon,\beta,n}(x_i)=\hat{d}_{\exnormd,\epsilon,\beta,n}(x_i)+\mathcal{O}\left(\frac{\sqrt{\log(n)}}{n^{1/2}\epsilon^{d/4}}\right)
\end{equation*}
Hence, we have with probability $1-\mathcal{O}(n^{-2})$,
\begin{align*}
\epsilon[L\boldsymbol{f}](i)=&\frac{\frac{1}{n}\sum_{j=1}^n K_{\epsilon,\beta,m}^{(R)}(x_i, x_j)d_{\exnormd,\epsilon,\beta,n}(x_j)^{-\alpha} (\mathbf{X}[i]-\Omega_{ij}\mathbf{X}[j]) }{\frac{1}{n}\sum_{j=1}^n K_{\epsilon,\beta,m}^{(R)}(x_i, x_j)d_{\exnormd,\epsilon,\beta,n}(x_j)^{-\alpha} }\\
&+\mathcal{O}(\epsilon^{3/2})+\mathcal{O}\left(\frac{\sqrt{\log(n)}}{n^{\gamma/2}\epsilon^{d/4}}\right)
\end{align*}
Now, we focus on the numerator. Let the random variables
\begin{equation*}
F:=K_{\epsilon,\beta,m}^{(R)}(x_i, x_j)d_{\exnormd,\epsilon,\beta,n}(x_j)^{-\alpha} (u^{-1}_{i}X(x_i)-u_i^{-1}\parallelslant_{Y}^{x_i}X(Y))
\end{equation*}
where $Y\sim p$. By Equation \eqref{eq:nume}
\begin{align*}
\mathbb{E}[F]&=u^{-1}_i\bigg[\epsilon\frac{\mu_{2,0,1}}{d} q_{\beta}(x_i) \rho_{\beta,\alpha}(x_i) \nabla^2\left(X(x)-X(x_i)\right)|_{x=x_i} \\
&\quad+\epsilon \frac{2\mu_{2,0,1}}{d} \left(q_{\beta}(x_i) \nabla \rho_{\beta,\alpha}(x_i)+ \rho_{\beta,\alpha}(x_i) \nabla q_{\beta}(x_i)\right)\cdot\nabla^{\E} \left(X(x)-X(x_i)\right)|_{x=x_i} \bigg]\\
\end{align*}
For the second moment, we focus on the $\ell$-th entry for $\ell=1,\cdots,q$:
\begin{align*}
&\mathbb{E}[F(\ell)^2]\\
&=\epsilon^{-2d}\int_{\M}K_{\epsilon,\beta,m}^{(R)}(x_i, x_j)^2d_{\exnormd,\epsilon,\beta,n}(x_j)^{-2\alpha} (u^{-1}_{i}X(x_i)-u_i^{-1}\parallelslant_{Y}^{x_i}X(Y))(\ell)^2p(y)dV(y)\\
&=\epsilon^{-d} \int \epsilon^{-\frac{d}{2}} A_0^{2}(v) q_\Z^2(x_i)\left[u_i^{-1} \nabla_v X\left(x_i\right)\right](\ell)^2 p(x_i)^{1-2\alpha}d v\\
&=2 \epsilon^{1-d} q_\Z^2(x)\left(x_i\right) p(x_i)^{1-2\alpha} \frac{\mu_{2,0,1}\mu_{2,0,2}}{d} \sum_{j=1}^d\left(u_i^{-1}\nabla_{E_j} X\left(x_i\right)\right)(\ell)^2
\end{align*}
By the similar work, we can control the denominator. Thus, with probability $1-\mathcal{O}(n^{-2})$, we have
\begin{align}\label{eq:event_4}
&\left[\left(I_{nq}-\left(\mathbf{D}_{\beta,\alpha}^{(R)}\right)^{-1} \mathbf{S}_{\beta,\alpha}^{(R)}\right) \mathbf{X}\right][i]\\
=&\,u^{-1}_i\left[X(x_i)-T_{\epsilon,\beta,\alpha}^{(R)}X(x_i)\right]+\mathcal{O}\left(\frac{\sqrt{\log n}}{n^{\gamma/2}\epsilon^{d/4-1/2}}\right)+\mathcal{O}(\epsilon^{3/2})\,,\nonumber
\end{align}
where the implied constant in the first term depends on $\sum_{j=1}^d\left(u_i^{-1}\nabla_{E_j} X\left(x_i\right)\right)(l)^2$ for all $\ell=1,\cdots,q$.
\end{proof}

\subsection{Proof of Proposition \ref{prop:eff}}
In Lemma \ref{lem:refkernel}, we analyzed the leading term of the effective kernel without normalized setting, corresponding to the case $\alpha=\beta=0$. Here, we consider the normalization factors. Following the calculations in \cite{shen2022}, we show Proposition \ref{prop:eff}.

\begin{proof}
We assume the kernel is Gaussian (see Equation \eqref{eq:kernel}). We can calculate directly,

\begin{equation*}
K_\epsilon\left(x,z\right) K_\epsilon\left(z,y\right)= e^{-\frac{2\left\|z-\left(x+y\right) / 2\right\|^2}{\epsilon}} e^{-\frac{\left\|x-y\right\|^2}{2 \epsilon}}\,.
\end{equation*}

Now, we consider the effective kernel,

\begin{equation*}
\begin{aligned}
K_{\epsilon,\beta,\alpha}\left(x,y\right) & =\int_{\M} K_\epsilon\left(x,z\right) K_\epsilon\left(z, y\right) p_\Z(z) d V(z) \\
& =e^{-\frac{\left\|x-y\right\|^2}{2 \epsilon}} \frac{1}{d_{\mathcal{X},\epsilon,\beta}(x)^{\alpha}d_{\mathcal{X},\epsilon,\beta}(y)^{\alpha}}\int_{\M}e^{-\frac{2\left\|z-\left(x+y\right) / 2\right\|^2}{\epsilon}} \frac{ p_{\Z}(z)}{d_{\Z,\epsilon}(z)^\beta} d V(z)\\
&=\epsilon^{-d\alpha+d\beta\alpha-d\beta+d/2}c_{\beta,\alpha}(x,y)e^{-\frac{\left\|x-y\right\|^2}{2 \epsilon}}
\end{aligned}
\end{equation*}
where $c_{\beta,\alpha}$ only depends on depends on $x$ and $y$ and is of order 1 since $\M$ is compact.

Next, let $u\in E_y$. By Lemma \ref{lem:double_para}, we have
\begin{equation*}
\begin{aligned}
S_{\epsilon,\beta,\alpha}\left(x,y\right)u & =\int_{\M} K_\epsilon\left(x,z\right) K_\epsilon\left(z, y\right) \parallelslant_z^x\parallelslant_y^z u\, p_\Z(z) d V(z) \\
& =\parallelslant_y^x u\int_{\M} K_\epsilon\left(x,z\right) K_\epsilon\left(z, y\right)\, p_\Z(z) d V(z) +\mathcal{O}(\epsilon^3)\\
& =\epsilon^{-d\alpha+d\beta\alpha-d\beta+d/2}c_{\beta,\alpha}(x,y)e^{-\frac{\left\|x-y\right\|^2}{2 \epsilon}}\parallelslant_y^x u +\mathcal{O}(\epsilon^3)\,.
\end{aligned}
\end{equation*}
\end{proof}


\end{document}